\documentclass{article}

\usepackage[final]{neurips_2022}

\usepackage{xcolor}         

\usepackage[utf8]{inputenc} 
\usepackage[T1]{fontenc}    
\usepackage{hyperref}       
\usepackage{url}            
\usepackage{booktabs}       
\usepackage{amsfonts}       
\usepackage{nicefrac}       
\usepackage{microtype}      
\usepackage{xcolor}         
\usepackage{longtable}
\usepackage{booktabs}
\usepackage{graphicx}

\usepackage{subfigure}
\usepackage{caption}
\usepackage{textcomp}

\usepackage{booktabs}
\usepackage{multirow}
\usepackage{wrapfig}
\usepackage{graphicx}
\usepackage{enumitem}
\usepackage{soul}
\usepackage{amsmath,amscd,amsbsy,amssymb,amsfonts,latexsym,url,bm,amsthm}
\def\BibTeX{{\rm B\kern-.05em{\sc i\kern-.025em b}\kern-.08em
		T\kern-.1667em\lower.7ex\hbox{E}\kern-.125emX}}
\usepackage{algorithm}
\usepackage{algorithmic}
\usepackage[algo2e,ruled,vlined,boxed,linesnumbered]{algorithm2e}
\usepackage{threeparttable}
\usepackage{listings}

\usepackage{colortbl}

\usepackage{multirow}
\usepackage{multicol}

\usepackage{natbib}
\setcitestyle{numbers,square}

\newtheorem{theorem}{Theorem}
\newtheorem{lemma}{Proposition}

\newtheorem{proposition}{Proposition}

\newcommand{\model}{\textsc{NodeFormer}\xspace}

\newcommand{\std}[1]{{\scriptsize #1}}

\usepackage{makecell}

\title{Align All the Nodes: Scalable Learning of Flexible Latent Structure for Graph Neural Networks}
\title{Uniting All the Others for One's Own Prediction: \\Scalable and Flexible Graph Structure Learning}
\title{Towards Scalable (All-Pair) Message Passing for \\Node Classification beyond Explicit Topology}
\title{\model: Scaling Graph Structure Learning for Efficient Transformer on Large Graphs}
\title{NodeFormer: Scaling Arbitrary Structure Learning for Efficient Transformer on Large Graphs}

\title{NodeFormer: Structure Learning without Explicit Aggregation for Efficient Transformer on Large Graphs}
\title{NodeFormer: Linearizing Transformer on Large Graphs with Scalable Latent Structure Learning}
\title{NodeFormer: A Scalable Graph Structure Learning Transformer for Node Classification}
\author{
  Qitian Wu$^{1}$, Wentao Zhao$^{1}$, Zenan Li$^{1}$, David Wipf$^2$, Junchi Yan$^{1}$\thanks{Wentao Zhao and Zenan Li contribute equally. The SJTU authors are also with MoE Key Lab of Artificial Intelligence, SJTU. Junchi Yan is the correspondence author who is also with Shanghai AI Laboratory.} \\
  $^{1}$Department of Computer Science and Engineering,
  Shanghai Jiao Tong University\\
  $^2$Amazon Web Service, Shanghai AI Lab \\
  \texttt{\{echo740,permanent,emiyali,yanjunchi\}@sjtu.edu.cn, davidwipf@gmail.com} 
}

\begin{document}

\maketitle

\begin{abstract}
    Graph neural networks have been extensively studied for learning with inter-connected data. Despite this, recent evidence has revealed GNNs' deficiencies related to over-squashing, heterophily, handling long-range dependencies, edge incompleteness and particularly, the absence of graphs altogether. While a plausible solution is to learn new adaptive topology for message passing, issues concerning quadratic complexity hinder simultaneous guarantees for scalability and precision in large networks. In this paper, we introduce a novel all-pair message passing scheme for efficiently propagating node signals between arbitrary nodes, as an important building block for a pioneering Transformer-style network for node classification on large graphs, dubbed as \textsc{NodeFormer}. Specifically, the efficient computation is enabled by a kernerlized Gumbel-Softmax operator that reduces the algorithmic complexity to linearity w.r.t. node numbers for learning latent graph structures from large, potentially fully-connected graphs in a differentiable manner. We also provide accompanying theory as justification for our design. Extensive experiments demonstrate the promising efficacy of the method in various tasks including node classification on graphs (with up to 2M nodes) and graph-enhanced applications (e.g., image classification) where input graphs are missing. The codes are available at \url{https://github.com/qitianwu/NodeFormer}.
\end{abstract}

\section{Introduction}




Relational structure inter-connecting instance nodes as a graph is ubiquitous from social domains (e.g., citation networks) to natural science (protein-protein interaction), where graph neural networks (GNNs)~\cite{scarselli2008gnnearly,GCN-vallina,graphsage,GAT} have shown promising power for leveraging such data dependence as geometric priors.
However, there arises increasing evidence challenging the core GNN hypothesis that propagating information along observed graph structures will necessarily produce better node-level representations for prediction on each individual instance node. Conflicts with this premise lead to commonly identified deficiencies with GNN message-passing rules w.r.t.~heterophily~\cite{h2gcn-neurips20}, over-squashing~\cite{oversquashing-iclr21}, long-range dependencies~\cite{longrange-icml18}, and graph incompleteness~\cite{IDLS-icml19}, etc.

Moreover, in graph-enhanced applications, e.g., text classification~\cite{text-aaai20}, vision navigation~\cite{image-cvpr21}, physics simulation~\cite{physics-icml20}, etc., graph structures are often unavailable though individual instances are strongly inter-correlated. A common practice is to artificially construct a graph via some predefined rules (e.g., $k$-NN), which is agnostic to downstream tasks and may presumably cause the misspecification of GNNs' inductive bias on input geometry (induced by the local feature propagation design).

Natural solutions resort to organically combining learning optimal graph topology with message passing. However, one critical difficulty is the \emph{scalability} issue with $O(N^2)$ (where $N$ denotes \#nodes) computational complexity, which is prohibitive for large networks (with $10K\sim 1M$ nodes). Some existing approaches harness neighbor sampling~\cite{Bayesstruct-aaai19}, anchor-based adjacency surrogates~\cite{IDGL-neurips20} and hashing schemes~\cite{kernalstruct-cikm18} to reduce the overhead; however, these strategies may sacrifice model precision and still struggle to handle graphs with million-level nodes. Another obstacle lies in the increased degrees of freedom due to at least an $N \times N$ all-pair similarity matrix, which may result in large combinatorial search space and vulnerability to over-fitting.


In this work, we introduce a novel all-pair message passing scheme that can scale to large systems without compromising performance. We develop a kernelized Gumbel-Softmax operator that seamlessly synthesizes \emph{random feature map}~\cite{rff} and approximated sampling strategy~\cite{gumbel-iclr17}, for distilling latent structures among all the instance nodes and yielding moderate gradients through differentiable optimization. Though such a combination of two operations involving randomness could potentially result in mutual distortion, we theoretically prove that the new operator can still guarantee a well-posed approximation for concrete variables (discrete structures) with the error bounded by feature dimensions. Furthermore, such a design can reduce the algorithmic complexity of learning new topology per layer to $O(N)$ by avoiding explicit computation for the cumbersome all-pair similarity. 



The proposed module opens the door to a new class of graph networks, i.e., \model (\emph{Scalable Transformers for Node Classification}), that is capable of efficiently propagating messages between arbitrary node pairs in flexible layer-specific latent graphs. And to accommodate input graphs (if any), we devise two simple techniques: a relational bias and an edge-level regularization loss, as guidance for properly learning adaptive structures. 
We evaluate our approach on diverse node classification tasks ranging from citation networks to images/texts. The results show its promising power for tackling heterophily, long-range dependencies, large-scale graphs, graph incompleteness and the absence of input graphs. 
The contributions of this paper are summarized as follows:

\textbf{$\bullet$\; } We develop a kernelized Gumbel-Softmax operator which is proven to serve as a well-posed approximation for concrete variables, particularly the discrete latent structure among data points. The new module can reduce the algorithmic complexity for learning new message-passing topology from quadratic to linear w.r.t. node numbers, without sacrificing the precision. This serves as a pioneering model that successfully scales graph structure learning to large graphs with million-level nodes. 

\textbf{$\bullet$\; } We further propose \model, a new class of graph networks with layer-wise message passing as operated over latent graphs potentially connecting all nodes. The latter are optimized in an end-to-end differentiable fashion through a new objective that essentially pursues sampling optimal topology from a posterior conditioned on node features and labels. To our knowledge, \model is the first Transformer model that scales all-pair message passing to large node classification graphs.

\textbf{$\bullet$\; } We demonstrate the model's efficacy by extensive experiments over a diverse set of datasets, including node classification benchmarks and image/text classification, where significant improvement over strong GNN models and SOTA structure learning methods is shown. Besides, it successfully scales to large graph datasets with up to 2M nodes where prior arts failed, and reduces the time/space consumption of the competitors by up to 93.1\%/80.6\% on moderate sized datasets.

\section{Related Works}\label{sec-related}

\textbf{Graph Neural Networks.} Building expressive GNNs is a fundamental problem in learning over graph data. With Graph Attention Networks (GAT)~\cite{GAT} as an early attempt, there are many follow-up works, e.g., \cite{learn2drop-wsdm20,GIB-neurips20}, considering weighting the edges in input graph for enhancing the expressiveness. Other studies, e.g., \cite{dropedge-iclr20, neuralsparse-icml20} focus on sparsifying input structures to promote robust representations. There are also quite a few approaches that propose scalable GNNs through, e.g., subgraph sampling~\cite{zeng2019graphsaint}, linear feature mapping~\cite{SGC-icml19}, and channel-wise transformation~\cite{zhang2022scalegcn}, etc. 
However, these works cannot learn new edges out of the scope of input geometry, which may limit the model's receptive fields within local neighbors and neglect global information.

\textbf{Graph Structure Learning.} Going beyond observed topology, graph structure learning targets learning a new graph for message passing among all the instances~\cite{GSL-survey}. One line of work is similarity-driven where the confidence of edges are reflected by some similarity functions between node pairs, e.g., Gaussian kernels~\cite{kernalstruct-cikm18}, cosine similarity~\cite{IDGL-neurips20}, attention networks~\cite{glconv-cvpr19}, non-linear MLP~\cite{patient-2020} etc. Another line of work optimizes the adjacency matrix. Due to the increased optimization difficulties, some sophisticated training methods are introduced, such as bi-level optimization~\cite{IDLS-icml19}, variational approaches~\cite{varstruct-neurips20,wsgnn}, Bayesian inference~\cite{Bayesstruct-aaai19} and projected gradient descent~\cite{structnorm-kdd20}. To push further the limits of structure learning, this paper proposes a new model \model (for enabling scalable node-level Transformers) whose merits are highlighted via a high-level comparison in Table~\ref{tbl-comparison}. In particular, \model enables efficient structure learning in each layer, does not require input graphs and successfully scales to graphs with 2M nodes. 

\textbf{Node-Level v.s. Graph-Level Prediction.} We emphasize upfront that our focus is on \textit{node-level} prediction tasks involving a single large graph such that scalability is paramount, especially if we are to consider arbitrary relationships across \textit{all} nodes (each node is an instance with label and one can treat all the nodes non-i.i.d. generated due to the inter-dependence) for structure-learning purposes. Critically though, this scenario is quite distinct from \textit{graph-level} classification tasks whereby each i.i.d. instance is itself a small graph and fully connecting nodes \textit{within} each graph is computationally inexpensive. While this latter scenario has been explored in the context of graph structure learning~\cite{pointcloud-19} and all-pair message passing design, e.g., graph Transformers~\cite{graphtransformer-2020}, existing efforts do not scale to the large graphs endemic to node-level prediction. 


\begin{table*}[tb!]
	\centering
	\caption{Comparison of popular graph structure learning approaches for \emph{node-level tasks} where in particular, the graph connects all instance nodes and one's target is for prediction on each individual node. For \emph{parameterization}, `Function' means learning through functional mapping and `Adjacency' means directly optimizing graph adjacency. For \emph{expressivty}, `Fixed' means learning one graph shared by all propagation layers and `Layer-wise' means learning graph structures per layers. The \emph{largest demo} means the largest \# nodes of datasets used. $\dagger$ $m$ denotes \# anchors (i.e., a subset of nodes).
}
 \vspace{-5pt}
	\label{tbl-comparison}
	\scriptsize
	\resizebox{0.99\textwidth}{!}{
	\begin{threeparttable}
    {{
    		\begin{tabular}{l|cccccc}
			\toprule[0.8pt]
			\textbf{Models}  & \textbf{Parameterization} & \textbf{Expressivity} & \textbf{Input Graphs} & \textbf{Inductive}  & \textbf{Complexity} & \textbf{Largest Demo} \\
			\midrule[0.6pt]
            LDS-GNN~\cite{IDLS-icml19}  & Adjacency & Fixed & Required & No  & $O(N^2)$ & 0.01M  \\
            ProGNN~\cite{structnorm-kdd20}  & Adjacency & Fixed & Required & No & $O(N^2)$ & 0.02M \\
            VGCN~\cite{varstruct-neurips20} & Adjacency & Fixed & Required & No  & $O(N^2)$ & 0.02M \\
            BGCN~\cite{Bayesstruct-aaai19}  & Adjacency & Fixed & Required & No  & $O(N^2)$ & 0.02M \\
            GLCN~\cite{glconv-cvpr19}  & Function & Fixed & Not necessary & Yes  & $O(N^2)$ & 0.02M \\
            IDGL~\cite{IDGL-neurips20}  & Function & Fixed & Required & Yes & $O(N^2)$ or $O(Nm)^\dagger$ & 0.1M \\
            \midrule
            \model (Ours)    &   Function & Layer-wise & Not necessary & Yes & $O(N)$ or $O(E)$ & 2M \\
			\bottomrule[0.8pt]
		\end{tabular}
		}
		}
		
	\end{threeparttable}
	}
 \vspace{-15pt}
\end{table*}

\section{\model: A Transformer Graph Network at Scale}

Let $\mathcal G = (\mathcal N, \mathcal E)$ denote a graph with $\mathcal N$ a node set ($|\mathcal N| = N$) and $\mathcal E \subseteq \mathcal N \times \mathcal N$ an edge set ($|\mathcal E| = E$). Each node $u\in \mathcal N$ is assigned with node features $\mathbf x_u \in \mathbb R^D$ and a label $y_u$. We define an adjacency matrix $\mathbf A = \{a_{uv}\}\in \{0, 1\}^{N\times N}$ where $a_{uv}=1$ if edge $(u,v)\in \mathcal E$ and $a_{uv}=0$ otherwise. Without loss of generality, $\mathcal E$ could be an empty set in case of no input structure. There are two common settings: transductive learning, where testing nodes are within the graph used for training, and inductive learning which handles new unseen nodes out of the training graph. The target is to learn a function for node-level prediction, i.e., estimate labels for unlabeled or new nodes in the graph.



\textbf{General Model and Key Challenges.} We start with the observation that the input structures may not be the ideal one for propagating signals among nodes and instead there exist certain latent structures that could facilitate learning better node representations. We thus consider the updating rule
\begin{equation}\label{eqn-update}
\tilde{\mathbf A}^{(l)} = g(\mathbf A, \mathbf Z^{(l)}; \omega), ~~~~~\mathbf Z^{(l+1)} = h(\tilde{\mathbf A}^{(l)}, \mathbf A, \mathbf Z^{(l)}; \theta),
\end{equation}
where $\mathbf Z^{(l)} = \{\mathbf z^{(l)}_u\}_{u\in \mathcal N}$ and $\tilde{\mathbf A}^{(l)} = \{\tilde a^{(l)}_{uv}\}_{u,v\in \mathcal N}$ denotes the node representations and the estimated latent graph of the $l$-th layer, respectively, and $g$, $h$ are both differentiable functions aiming at 1) structure estimation for a layer-specific latent graph $\tilde{\mathbf A}^{(l)}$ based on node representations and 2) feature propagation for updating node representations, respectively. 
The model defined by Eqn.~\ref{eqn-update} follows the spirit of Transformers~\cite{vaswani2017attention} (where in particular $\tilde{\mathbf A}^{(l)}$ can be seen as an attentive graph) that potentially enables message passing between any node pair in each layer, which, however, poses two \emph{challenges}: 

\textbf{$\bullet$ (Scalability)}: How to reduce the prohibitive quadratic complexity for learning new graphs?

\textbf{$\bullet$ (Differentiability)}: How to enable end-to-end differentiable optimization for discrete structures?

Notice that the first challenge is non-trivial in node-level prediction tasks (the focus of our paper), since the latent graphs could potentially connect \emph{all the instance nodes} (e.g., from thousands to millions, depending on dataset sizes), which is fairly hard to guarantee both precision and scalability. 


\subsection{Efficient Learning Discrete Structures}\label{sec-model-gnn}

We describe our new message-passing scheme with an efficient kernelized Gumbel-Softmax operator to resolve the aforementioned challenges.
We assume $\mathbf z_u^{(0)} = \mathbf x_u$ as the initial node representation. 

\textbf{Kernelized Message Passing.} We define a full-graph attentive network that estimates latent interactions among instance nodes and enables corresponding densely-connected message passing:
\begin{equation}\label{eqn-attn-ori}
    \tilde a_{uv}^{(l)} = \frac{\exp((W_Q^{(l)} \mathbf z_u^{(l)})^\top (W_K^{(l)} \mathbf z_v^{(l)}))}{\sum_{w=1}^N \exp((W_Q^{(l)} \mathbf z_u^{(l)})^\top (W_K^{(l)} \mathbf z_w^{(l)}))}, \quad \mathbf z_u^{(l+1)} = \sum_{v=1}^N \tilde a_{uv}^{(l)} \cdot (W_V^{(l)} \mathbf z_v^{(l)}), 
\end{equation}
where $W_Q^{(l)}$, $W_K^{(l)}$ and $W_V^{(l)}$ are learnable parameters in $l$-th layer. We omit non-linearity activation (after aggregation) for brevity. The updating for $N$ nodes in one layer using Eqn.~\ref{eqn-attn-ori} requires prohibitive $\mathcal O(N^2)$ complexity. 
Also, given large $N$, the normalization in the denominator would shrink attention weights to zero and lead to gradient vanishing. We call this problem as \emph{over-normalizing}.

To accelerate the full-graph model, we observe that the \emph{dot-then-exponentiate} operation in Eqn.~\ref{eqn-attn-ori} can be converted into a pairwise similarity function:
\begin{equation}\label{eqn-attn-kernel}
    \mathbf z_u^{(l+1)} = \sum_{v=1}^N \frac{\kappa(W_Q^{(l)} \mathbf z_u^{(l)}, W_K^{(l)} \mathbf z_v^{(l)})}{\sum_{w=1}^N \kappa(W_Q^{(l)} \mathbf z_u^{(l)}, W_K^{(l)} \mathbf z_w^{(l)})} \cdot (W_V^{(l)} \mathbf z_v^{(l)}),
\end{equation}
where $\kappa(\cdot, \cdot): \mathbb R^d \times \mathbb R^d \rightarrow \mathbb R$ is a positive-deﬁnite kernel measuring the pairwise similarity. The kernel function can be further approximated by random features (RF)~\cite{rff}
which serves as an unbiased estimation via $\kappa(\mathbf a, \mathbf b)= \langle\Phi(\mathbf a),\Phi(\mathbf b)\rangle_{\mathcal V} \approx \phi(\mathbf a)^\top\phi(\mathbf b)$, where the first equation is by Mercer’s theorem with $\Phi: \mathbb R^d \rightarrow \mathcal V$ a basis function and $\mathcal V$ a high-dimensional vector space, and $\phi(\cdot): \mathbb R^d \rightarrow \mathbb R^m$ is a low-dimensional feature map with random transformation. There are many potential choices for $\phi$, e.g., Positive Random Features (PRF)~\cite{performer-iclr21}
\begin{equation}\label{eqn-pos-rff}
    \phi(\mathbf x) = \frac{\exp{(\frac{-\|\mathbf x\|_2^2}{2}})}{\sqrt{m}} [\exp(\mathbf w_1^\top \mathbf x), \cdots, \exp(\mathbf w_m^\top \mathbf x)],
\end{equation}
where $\mathbf w_k\sim \mathcal N(0, I_d)$ is i.i.d. sampled random transformation. 
The RF converts dot-then-exponentiate operation into inner-product in vector space, which enables us to re-write Eqn.~\ref{eqn-attn-kernel} (assuming $\mathbf q_u = W_Q^{(l)} \mathbf z_u^{(l)}$, $\mathbf k_u = W_K^{(l)} \mathbf z_u^{(l)}$ and $\mathbf v_u = W_V^{(l)} \mathbf z_u^{(l)}$ for simplicity):
\begin{equation}\label{eqn-attn-rff}
        \mathbf z_u^{(l+1)}  = \sum_{v=1}^N \frac{\phi(\mathbf q_u)^\top \phi(\mathbf k_v)}{\sum_{w=1}^N \phi(\mathbf q_u)^\top \phi(\mathbf k_w)} \cdot \mathbf v_v = \frac{\phi(\mathbf q_u)^\top  \sum_{v=1}^N \phi(\mathbf k_v) \cdot \mathbf v_v^\top  }{ \phi(\mathbf q_u)^\top \sum_{w=1}^N \phi(\mathbf k_w)}.
\end{equation}
The key advantage of Eqn.~\ref{eqn-attn-rff} is that the two summations are shared by each $u$, so that one only needs to compute them once and re-used for others. Such a property enables $\mathcal O(N)$ computational complexity for full-graph message passing, which paves the way for learning graph structures among large-scale instances. Moreover, one can notice that Eqn.~\ref{eqn-attn-rff} avoids computing the $N\times N$ similarity matrix, i.e., $\{\tilde a_{uv}^{(l)}\}_{N\times N}$, required by Eqn.~\ref{eqn-attn-ori}, thus also reducing the learning difficulties.

Nevertheless, Eqn.~\ref{eqn-attn-rff} still suffers what we mentioned the over-normalizing issue. The crux is that the message passing is operated on a weighted fully-connected graph where, in fact, only partial edges are important. Also, such a deterministic way of feature aggregation over all the instances may increase the risk for over-fitting, especially when $N$ is large. 
We next resolve the issues by distilling a sparse structure from the fully-connected graph. 

\textbf{Differentiable Stochastic Structure Learning.} The difficultly lies in how to enable differentiable optimization for discrete graph structures. The weight $\tilde a_{uv}^{(l)}$ given by Eqn.~\ref{eqn-attn-ori} could be used to define a categorical distribution for generating latent edges from distribution $\mbox{Cat}(\bm\pi_u^{(l)})$ where $\bm\pi_u^{(l)} = \{\pi_{uv}^{(l)}\}_{v=1}^N$ and $\pi_{uv}^{(l)} = p(v|u) = \tilde a_{uv}^{(l)}$. Then in principle, we can sample over the categorical distribution multiple times for each node to obtain its neighbors. However, the sampling process would introduce discontinuity and hinders back-propagation. Fortunately, we notice that the Eqn.~\ref{eqn-attn-kernel} can be modified to incorporate the reparametrization trick~\cite{gumbel-iclr17} to allow differentiable learning:
\begin{equation}\label{eqn-attn-gumbel-ori}
    \mathbf z_u^{(l+1)}  = \sum_{v=1}^N \frac{\exp((\mathbf q_u^\top \mathbf k_v + g_v) /\tau)}{\sum_{w=1}^N \exp((\mathbf q_u^\top \mathbf k_w + g_w)/\tau)} \cdot \mathbf v_v = \sum_{v=1}^N \frac{\kappa(\mathbf q_u / \sqrt{\tau}, \mathbf k_v / \sqrt{\tau}) e^{g_v/\tau}}{\sum_{w=1}^N \kappa(\mathbf q_u / \sqrt{\tau}, \mathbf k_w / \sqrt{\tau}) e^{g_w/\tau}} \cdot \mathbf v_v,
\end{equation}
where $g_u$ is i.i.d. sampled from Gumbel distribution and $\tau$ is a temperature coefficient. Eqn.~\ref{eqn-attn-gumbel-ori} is a continuous relaxation of sampling one neighbored node for $u$ over $\mbox{Cat}(\bm \pi_u^{(l)})$ and $\tau$ controls the closeness to hard discrete samples~\cite{Concrete-iclr17}. Following similar reasoning as Eqn.~\ref{eqn-attn-kernel} and \ref{eqn-attn-rff}, we can yield
\begin{equation}\label{eqn-attn-gumbel-final}
        \mathbf z_u^{(l+1)} \approx \sum_{v=1}^N \frac{\phi(\mathbf q_u / \sqrt{\tau})^\top \phi(\mathbf k_v / \sqrt{\tau}) e^{g_v/\tau}}{\sum_{w=1}^N \phi(\mathbf q_u / \sqrt{\tau})^\top \phi(\mathbf k_w / \sqrt{\tau}) e^{g_w/\tau}} \cdot \mathbf v_v  = \frac{\phi(\mathbf q_u / \sqrt{\tau})^\top  \sum_{v=1}^N e^{g_v/\tau}\phi(\mathbf k_v / \sqrt{\tau}) \cdot \mathbf v_v^\top  }{ \phi(\mathbf q_u / \sqrt{\tau})^\top \sum_{w=1}^N e^{g_w/\tau}\phi(\mathbf k_w / \sqrt{\tau})}. 
\end{equation}
Eqn.~\ref{eqn-attn-gumbel-final} achieves message passing over a sampled latent graph (where we only sample once for each node) and still guarantees linear complexity as Eqn.~\ref{eqn-attn-rff}. In practice, we can sample $K$ times (e.g., $K= 5$) for each node and take an average of the aggregated results. Due to space limit, we defer more details concerning the differentiable sampling-based message passing to Appendix~\ref{appx-model}. Besides, in Fig.~\ref{fig-model} and Alg.~1 of Appendix~\ref{appx-model}, we present an illustration for node embedding updating in each layer, from a matrix view that is practically used for implementation.



\subsection{Well-posedness of the Kernelized Gumbel-Softmax Operator}\label{sec-model-theory}

One reasonable concern for Eqn.~\ref{eqn-attn-gumbel-final} is whether the RF approximation for kernel functions maintains the well-posedness of Gumbel approximation for the target discrete variables. As a justification for the new message-passing function, we next answer two theoretical questions: 1) How is the approximation capability of RF for the original dot-then-exponentiate operation with Gumbel variables in Eqn.~\ref{eqn-attn-gumbel-ori}? 2) Does Eqn.~\ref{eqn-attn-gumbel-final} still guarantee a continuous relaxation of the categorical distributions? We formulate the results as follows and defer proofs to Appendix~\ref{appx-proof}.
\begin{theorem}[Approximation Error for Softmax-Kernel]\label{thm-error}
    Assume $\|\mathbf q_u\|_2$ and $\|\mathbf k_v\|_2$ are bounded by $r$, then with probability at least $1-\epsilon$, the gap $\Delta = \left |\phi(\mathbf q_u / \sqrt{\tau})^\top \phi(\mathbf k_v / \sqrt{\tau}) - \kappa(\mathbf q_u / \sqrt{\tau}, \mathbf k_v / \sqrt{\tau}) \right |)$, where $\phi$ is defined by Eqn.~\ref{eqn-pos-rff}, will be bounded by
        $\mathcal O \left (\sqrt{\frac{\exp(6r/\tau)}{m\epsilon}} \right )$.
\end{theorem}
We can see that the error bound of RF for approximating original softmax-kernel function depends on both the dimension of feature map $\phi$ and temperature $\tau$. Notably, the error bound is independent of node number $N$, which implies that the approximation ability is insensitive to dataset sizes.

The second question is non-trivial since Eqn.~\ref{eqn-attn-gumbel-final} involves randomness of Gumbel variables and random transformation in $\phi$, which \emph{cannot} be decoupled apart. We define 
$c_{uv} = \frac{\phi(\mathbf q_u / \sqrt{\tau})^\top \phi(\mathbf k_v / \sqrt{\tau}) e^{g_v/\tau}}{\sum_{w=1}^N \phi(\mathbf q_u / \sqrt{\tau})^\top \phi(\mathbf k_w / \sqrt{\tau}) e^{g_w/\tau}}$
as the result from the kernelized Gumbel-Softmax and $\mathbf c_u = \{c_{uv}\}_{v=1}^N$ denotes the sampled edge vector for node $u$. We can arrive at the result as follows.
\begin{theorem}[Property of Kernelized Gumbel-Softmax Random Variables]\label{thm-gumbel}
     Suppose $m$ is sufficiently large, we have the convergence property for the kernelized Gumbel-Softmax operator
    \begin{equation}
        \lim_{\tau\rightarrow 0}\mathbb P(c_{uv} > c_{uv'}, \forall v'\neq v) = \frac{\exp(\mathbf q_u^\top \mathbf k_v)}{\sum_{w=1}^N \exp(\mathbf q_u^\top \mathbf k_w )},\quad \lim_{\tau\rightarrow 0} \mathbb P(c_{uv} = 1) = \frac{\exp(\mathbf q_u^\top \mathbf k_v )}{\sum_{w=1}^N \exp(\mathbf q_u^\top \mathbf k_w )}.\nonumber
    \end{equation}
\end{theorem}
It shows that when i) the dimension of feature map is large enough and ii) the temperature goes to zero, the distribution from which latent structures are sampled would converge to the original categorical distribution. 

\emph{Remark.} The two theorems imply a trade-off between RF approximation and Gumbel-Softmax approximation w.r.t. the choice of $\tau$. A large $\tau$ would help to reduce the burden on
kernel dimension $m$, and namely, small $\tau$ would require a very large $m$ to guarantee enough RF approximation precision. On the other hand, if $\tau$ is too large, the weight on each edge will converge to $\frac{1}{N}$, i.e., the model nearly degrades to mean pooling, while a small $\tau$ would endow the kernelized Gumbel-Softmax with better approximation to the categorical distribution. Empirical studies on this are presented in Appendix~\ref{appx-result}.

\subsection{Input Structures as Relational Bias}\label{sec-model-pos}

Eqn.~\ref{eqn-attn-gumbel-final} does not leverage any information from observed geometry which, however, is often recognized important for modeling physically-structured data~\cite{geometriclearning-2016}. We therefore accommodate input topology (if any) as relational bias via modifying the attention weight as
$\tilde a_{uv}^{(l)} \leftarrow \tilde a_{uv}^{(l)} + \mathbb I[a_{uv} = 1] \sigma(b^{(l)})$,
where $b^{(l)}$ is a learnable scalar as relational bias for any adjacent node pairs $(u, v)$ and $\sigma$ is a certain (bounded) activation function like sigmoid. The relational bias aims at assigning adjacent nodes in $\mathcal G$ with proper weights, 
and the node representations could be accordingly updated by
\begin{equation}\label{eqn-pos}
    \mathbf z_u^{(l+1)} \leftarrow \mathbf z_u^{(l+1)} + \sum_{v, a_{uv}=1} \sigma(b^{(l)}) \cdot \mathbf v_v.
\end{equation}
Eqn.~\ref{eqn-pos} increases the algorithmic complexity for message passing to $\mathcal O(N + E)$, albeit within the same order-of-magnitude as common GNNs operating on input graphs. Also, one can consider higher-order adjacency as relational bias for better expressiveness at some expense of efficiency, as similarly done by \cite{mixhop-icml19}. We summarize the feed-forward computation of \model in Alg.~1. 

\begin{figure*}[t!]
\centering
\centering
\includegraphics[width=\textwidth]{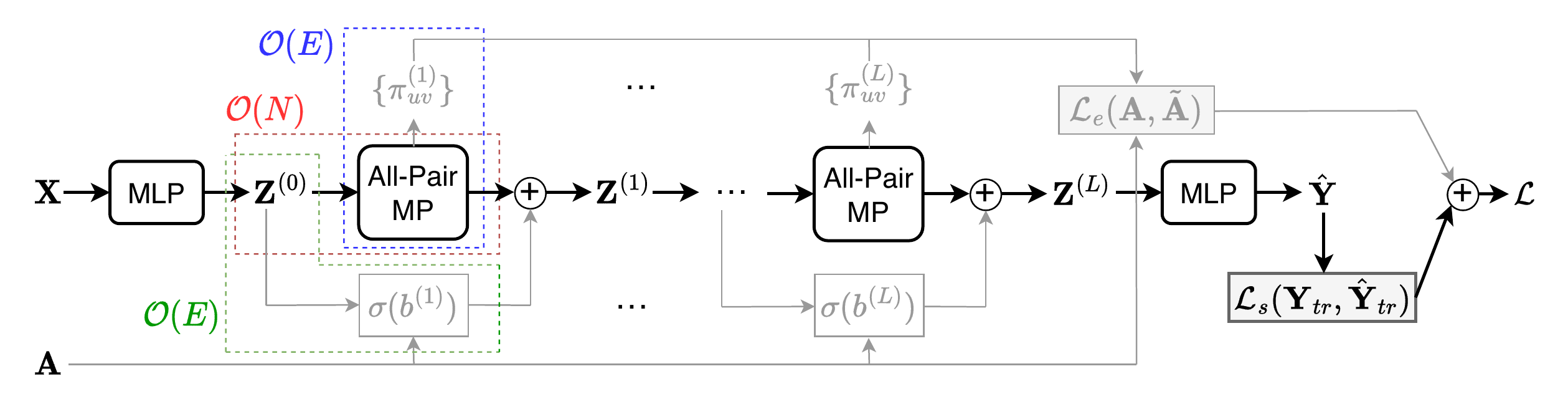}
\caption{Illustration for the data flow of \model which takes node embedding matrix $\mathbf X$ and (optional) graph adjacency matrix $\mathbf A$ as input. There are three components in \model. The first one is the all-pair message passing (MP) module (colored red) which adopts our proposed kernelized Gumbel-Softmax operator to update node embeddings in each layer with $\mathcal O(N)$ complexity. The other two components are optional based on the availability of input graphs: 1) relational bias (colored green) that reinforces the propagation weight on observed edges; 2) edge regularization loss (colored blue) that aims to maximize the probability for observed edges. These two components require $\mathcal O(E)$ complexity. The final training loss $\mathcal L$ is the weighted sum of the standard supervised classification loss and the edge regularization loss.}
\label{fig-dataflow}
\end{figure*}

\subsection{Learning Objective}\label{sec-training}

Given training labels $\mathbf Y_{tr} = \{y_u\}_{u\in \mathcal N_{tr}}$, where $\mathcal N_{tr}$ denotes the set of labeled nodes, the common practice is to maximize the observed data log-likelihood 
which yields a supervised loss (with $C$ classes)
\begin{equation}\label{eqn-loss-stand}
    \mathcal L_s(\mathbf Y_{tr}, \hat {\mathbf Y}_{tr}) = - \frac{1}{N_{tr}}\sum_{v\in \mathcal N_{tr}} \sum_{c=1}^C \mathbb I[y_u=c] \log \hat y_{u,c},
\end{equation}
where $\mathbb I[\cdot]$ is an indicator function. However, it may not suffice to generalize well due to that the graph topology learning increases the degrees of freedom and the number of training labels is not comparable to that. Therefore, we additionally introduce an edge-level regularization:
\begin{equation}\label{eqn-loss-mle}
    \mathcal L_{e}(\mathbf A, \tilde {\mathbf A}) = - \frac{1}{NL} \sum_{l=1}^L \sum_{(u, v)\in \mathcal E} \frac{1}{d_u} \log \pi_{uv}^{(l)},
\end{equation}
where $d_u$ denotes the in-degree of node $u$ and $\pi_{uv}^{(l)}$ is the predicted probability for edge $(u,v)$ at the $l$-th layer. Eqn.~\ref{eqn-loss-mle} is a maximum likelihood estimation for edges in $\mathcal E$, with data distribution defined
\begin{equation}\label{eqn-edge-prior}
    p_0(v|u) = \left\{ 
    \begin{array}{cc}
         & \frac{1}{d_u}, \quad a_{uv} = 1  \\
         & 0, \quad otherwise.
    \end{array}
    \right. 
\end{equation}
We next show how to efficiently obtain $\pi_{uv}^{(l)}$. Although the feed-forward \model computation defined by Eqn.~\ref{eqn-attn-gumbel-final} does not explicitly produce the value for each $\pi_{uv}^{(l)}$, we can query their values by
\begin{equation}
    \pi_{uv}^{(l)} = \frac{\phi(W_Q^{(l)} \mathbf z_u^{(l)})^\top  \phi(W_K^{(l)} \mathbf z_v^{(l)}) }{ \phi(W_Q^{(l)} \mathbf z_u^{(l)})^\top \sum_{w=1}^N \phi(W_K^{(l)} \mathbf z_w^{(l)})},
\end{equation}
where the summation term can be re-used from once computation, as is done by Eqn.~\ref{eqn-attn-rff} and Eqn.~\ref{eqn-attn-gumbel-final}. Therefore, after once computation for the summation that requires $\mathcal O(N)$, the computation for each $\pi_{uv}^{(l)}$ requires $\mathcal O(1)$ complexity, yielding the total complexity controlled within $\mathcal O(E)$ (since we only need to query the observed edges). The final objective can be the combination of two:
$\mathcal L = \mathcal L_s + \lambda \mathcal L_e$, where $\lambda$ controls how much emphasis is put on input topology. We depict the whole data flow of \model's training in Fig.~\ref{fig-dataflow}.

\begin{figure*}[!tb]
    \centering
    \includegraphics[width=0.95\textwidth]{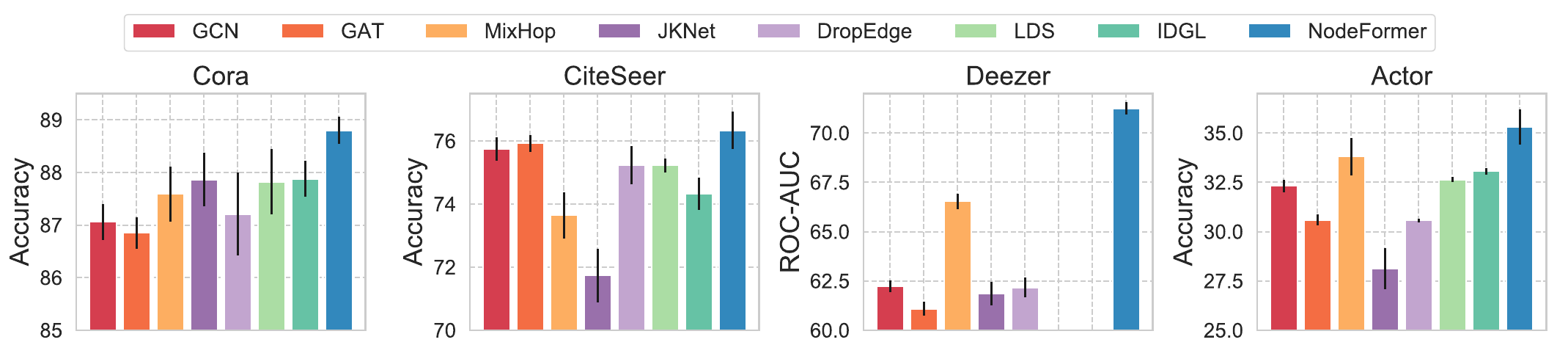}
    \caption{Experimental results for node classification in transductive setting on four common datasets. The missing results on \texttt{Deezer} is caused by out-of-memory (OOM).}
    \label{fig:transductive}
        \vspace{-5pt}
\end{figure*}

\begin{figure}[t]
	\begin{minipage}[t]{0.49\linewidth}
		\captionof{table}{Testing ROC-AUC and training memory cost on \texttt{OGB-Proteins} with batch size 10K.}
		\label{tab:proteins}
		\resizebox{0.95\linewidth}{!}{
		\begin{tabular}[t]{l|c|c}
            \toprule
			\textbf{Method} & \textbf{ROC-AUC (\%)} & \textbf{Train Mem} \\
            \midrule
			MLP & 72.04 \std{$\pm$ 0.48} & 2.0 GB \\
			GCN & 72.51 \std{$\pm$ 0.35} & 2.5 GB \\
                SGC & 70.31 \std{$\pm$ 0.23} & 1.2 GB \\
                GraphSAINT-GCN & 73.51 \std{$\pm$ 1.31} & 2.3 GB\\
                GraphSAINT-GAT & 74.63 \std{$\pm$ 1.24} & 5.2 GB\\
			\midrule
			\model & \cellcolor{lightgray}\textbf{77.45} \std{$\pm$ 1.15} & 3.2 GB \\
			\model-dt & 75.50 \std{$\pm$ 0.64} & 3.1 GB \\
			\model-tp & 76.18 \std{$\pm$ 0.09} & 3.2 GB \\
            \bottomrule
		\end{tabular}
	}
	\end{minipage}	
	\hspace{4pt}
	\begin{minipage}[t]{0.49\linewidth}
		\captionof{table}{Testing Accuracy and training memory cost on \texttt{Amazon2M} with batch size 100K.}
		\label{tab:amazon}
		\resizebox{0.95\linewidth}{!}{
		\begin{tabular}[t]{l|c|c}
            \toprule
			\textbf{Method} & \textbf{Accuracy (\%)} & \textbf{Train Mem} \\
            \midrule
			$\mbox{MLP}$ &  63.46 \std{$\pm$ 0.10}  & 1.4 GB  \\
			$\mbox{GCN}$ & 83.90 \std{$\pm$ 0.10} & 5.7 GB \\
                $\mbox{SGC}$ & 81.21 \std{$\pm$ 0.12} & 1.7 GB \\
                GraphSAINT-GCN & 83.84 \std{$\pm$ 0.42} &  2.1 GB\\
                GraphSAINT-GAT & 85.17 \std{$\pm$ 0.32} &  2.2 GB\\
			\midrule
		    $\mbox{\model}$ & \cellcolor{lightgray}\textbf{87.85} \std{$\pm$ 0.24} & 4.0 GB \\
			$\mbox{\model-dt}$ & 87.02 \std{$\pm$ 0.75} & 2.9 GB \\
			$\mbox{\model-tp}$ & 87.55 \std{$\pm$ 0.11} & 4.0 GB \\
            \bottomrule
		\end{tabular}
	}
	\end{minipage}	
\end{figure}



\section{Evaluation}\label{sec-exp}
We consider a diverse set of datasets for experiments and present detailed dataset information in Appendix~\ref{appx-dataset}. For implementation, 
we set $\sigma$ as sigmoid function and $\tau$ as 0.25 for all datasets. The output prediction layer is a one-layer MLP. More implementation details are presented in Appendix~\ref{appx-implementation}.
All experiments are conducted on a NVIDIA V100 with
16 GB memory.

As baseline models, we basically consider GCN~\cite{GCN-vallina} and GAT~\cite{GAT}. Besides, we compare with some advanced GNN models, including JKNet~\cite{jknet-icml18} and MixHop~\cite{mixhop-icml19}. 
These GNN models all rely on input graphs. 
We further consider DropEdge~\cite{dropedge-iclr20} and 
two SOTA graph structure learning methods, LDS-GNN~\cite{IDLS-icml19} and IDGL~\cite{IDGL-neurips20} for comparison. For large-scale datasets, we additionally compare with two scalable GNNs, a linear model SGC~\cite{SGC-icml19} and a graph-sampling model GraphSAINT~\cite{zeng2019graphsaint}. More detailed information about these models are presented in Appendix~\ref{appx-implementation}. 
All the experiments are repeated five times with different initializations.


\subsection{Experiments on Transductive Node Classification}\label{sec-exp-trans}
We study supervised node classification in transductive setting on common graph datasets: \texttt{Cora}, \texttt{Citeseer}, \texttt{Deezer} and \texttt{Actor}. 
The first two have high homophily ratios and the last two are identified as heterophilic graphs~\cite{h2gcn-neurips20,newbench}. These datasets are of small or medium sizes (with 2K$\sim$20K nodes). We use random splits with train/valid/test ratios as 50\%/25\%/25\%. For evaluation metrics, we use ROC-AUC for binary classification on \texttt{Deezer} and Accuracy for other datasets with more than 2 classes. Results are plotted in Fig.~\ref{fig:transductive} and \model achieves the best mean Accuracy/ROC-AUC across four datasets and in particular, outperforms other models by a large margin on two heterophilic graphs. The results indicate that \model can handle both homophilious and non-homophilious graphs. Compared with two structure learning models LDS and IDGL, \model yields significantly better performance, which shows its superiority. Also, for \texttt{Deezer}, LDS and IDGL suffers from out-of-memory (OOM). In fact, the major difficulty for \texttt{Deezer} is the large dimensions of input node features (nearly 30K), which causes OOM for IDGL even with the anchor approximation. In contrast, \model manages to scale and produce desirable accuracy.

\subsection{Experiments on Larger Graph Datasets}\label{sec-exp-large}

To further test the scalability, we consider two large-sized networks, \texttt{OGB-Proteins} and \texttt{Amazon2M}, with over 0.1 million and 2 million of nodes, respectively. 
\texttt{OGB-Proteins} is a multi-task dataset with 112 output dimensions, while \texttt{Amazon2M} is extracted from the Amazon Co-Purchasing network that entails long-range dependence~\cite{longrange-nips20}.
For \texttt{OGB-Proteins}, we use the protocol of \cite{ogb-nips20} and ROC-AUC for evaluation.
For \texttt{Amazon2M}, 
we adopt random splitting with 50\%/25\%/25\% nodes for training, validation and testing, respectively. Due to the large dataset size, we adopt mini-batch partition for training, in which case, for \model we only consider structure learning among nodes in a random mini-batch. We use batch size 10000 and 100000 for \texttt{Proteins} and \texttt{Amazon2M}, respectively. While the mini-batch partition may sacrifice the exposure to all instances, we found using large batch size can yield decent performance, which is also allowable thanks to the $\mathcal O(N)$ complexity of our model. For example, even setting the batch size as 100000, we found \model costs only 4GB GPU memory for training on \texttt{Amazon2M}. Table~\ref{tab:proteins} presents the results on \texttt{OGB-Proteins} where for fair comparison mini-batch training is also used for other models except GraphSAINT. We found that \model yields much better ROC-AUC and only requires comparable memory as simple GNN models. Table~\ref{tab:amazon} reports the results on \texttt{Amazon2M} which shows that \model outperforms baselines by a large margin and the memory cost is even fewer than GCN. This shows its practical efficacy and scalability on large-scale datasets and also the capability for addressing long-range dependence with shallow layers (we use $L=3$).

\begin{table*}[t!]
	\centering
	\caption{Experimental results on semi-supervised classficiation on \texttt{Mini-ImageNet} and \texttt{20News-Groups} where we use $k$-NN (with different $k$'s) for artificially constructing an input graph.}
	\label{tab:cv-nlp}
	\resizebox{\linewidth}{!}{
			\begin{tabular}{c|cccc|cccc}
				\toprule
				\multirow{2}{*}{\textbf{Method}} &  \multicolumn{4}{c|}{\texttt{Mini-ImageNet}}	 & \multicolumn{4}{c}{\texttt{20News-Group}} \\
				& $ k=5 $ & $ k=10 $ & $ k=15 $ & $ k=20 $  & $ k=5 $ & $ k=10 $ & $ k=15 $ & $ k=20 $ \\
				\midrule
				GCN & 84.86 \std{$\pm$ 0.42} & 85.61 \std{$\pm$ 0.40} & 85.93 \std{$\pm$ 0.59} & 85.96 \std{$\pm$ 0.66}& 65.98 \std{$\pm$ 0.68} & 64.13 \std{$\pm$ 0.88} & 62.95 \std{$\pm$ 0.70} & 62.59 \std{$\pm$ 0.62} \\
				GAT & 84.70 \std{$\pm$ 0.48} & 85.24 \std{$\pm$ 0.42}  & 85.41 \std{$\pm$ 0.43}  & 85.37 \std{$\pm$ 0.51} &  64.06 \std{$\pm$ 0.44}  & 62.51 \std{$\pm$ 0.71}  &  61.38 \std{$\pm$ 0.88} & 60.80 \std{$\pm$ 0.59} \\
				DropEdge & 83.91 \std{$\pm$ 0.24} & 85.35 \std{$\pm$ 0.44} & 85.25 \std{$\pm$ 0.63} & 85.81 \std{$\pm$ 0.65}& 64.46 \std{$\pm$ 0.43} & 64.01 \std{$\pm$ 0.42} & 62.46 \std{$\pm$ 0.51} & 62.68 \std{$\pm$ 0.71} \\
				IDGL & 83.63 \std{$\pm$ 0.32}  & 84.41 \std{$\pm$ 0.35}  &   85.50 \std{$\pm$ 0.24}  &  85.66 \std{$\pm$ 0.42}  &  65.09 \std{$\pm$ 1.23}  & 63.41 \std{$\pm$ 1.26}  & 61.57 \std{$\pm$ 0.52}  & 62.21 \std{$\pm$ 0.79}  \\
				LDS &OOM &OOM &OOM &OOM &  \cellcolor{lightgray}\textbf{66.15} \std{$\pm$ 0.36}  &  64.70 \std{$\pm$ 1.07}  & 63.51 \std{$\pm$ 0.64}  & 63.51 \std{$\pm$ 1.75}  \\
				\midrule
				\model &   \cellcolor{lightgray}\textbf{86.77} \std{$\pm$ 0.45}  &   \cellcolor{lightgray}\textbf{86.74} \std{$\pm$ 0.23}  &   \cellcolor{lightgray}\textbf{86.87} \std{$\pm$ 0.41}  &   \cellcolor{lightgray}\textbf{86.64} \std{$\pm$ 0.42} &  66.01 \std{$\pm$ 1.18}  &  \cellcolor{lightgray}\textbf{65.21} \std{$\pm$ 1.14}  &    \cellcolor{lightgray}\textbf{64.69} \std{$\pm$ 1.31}  &   \cellcolor{lightgray}\textbf{64.55} \std{$\pm$ 0.97}  \\
    \midrule
                \model w/o graph & \multicolumn{4}{c|}{\textbf{87.46}  \std{$\pm$ 0.36}} & \multicolumn{4}{c}{\textbf{64.71}  \std{$\pm$ 1.33}} \\
				\bottomrule
			\end{tabular}		}
\end{table*}

\begin{figure*}[!tb]
    \centering
    \includegraphics[width=0.95\textwidth]{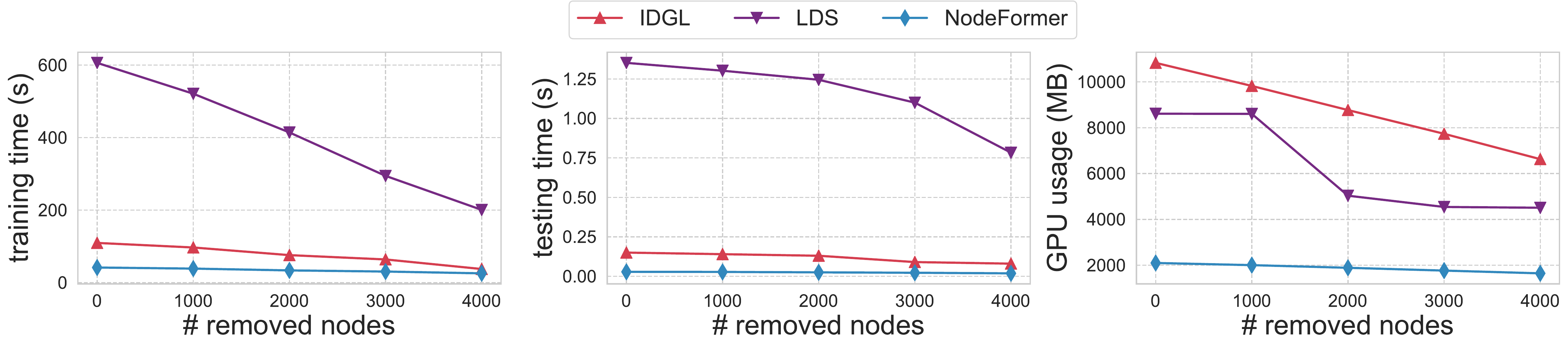}
    \caption{Comparison of training/inference time and GPU memory cost w.r.t. different instance numbers (by removing a certain portion of nodes) on \texttt{20News-Groups}.}
    \label{fig:timespace}
    \vspace{-5pt}
\end{figure*}

\subsection{Experiments on Graph-Enhanced Applications}\label{sec-exp-app}
\vspace{-5pt}
We apply our model to semi-supervised image and text classification on \texttt{Mini-ImageNet} and \texttt{20News-Groups} datasets, without input graphs. The instances of \texttt{Mini-ImageNet}~\cite{miniimagenet-neurips2016} are
84×84 RGB images and we randomly choose 30 classes each of which contains 600 samples for experiments.
\texttt{20News-Groups}~\cite{pedregosa2011scikit} consists of nearly 10K texts whose features are extracted by TF-IDF. More details for preprocessing are presented in Appendix~\ref{appx-dataset}. Also, for each dataset, we randomly split instances into 50\%/25\%/25\% for train/valid/test. Since there is no input graph, we use $k$-NN (over input node features) for artificially constructing a graph for enabling GNN's message passing and the graph-based components (edge regularization and relational bias) of \model. Table~\ref{tab:cv-nlp} presents the comparison results under different $k$'s. We can see that \model achieves the best performance in seven cases out of eight. The performance of GNN competitors varies significantly with different $k$ values, and \model is much less sensitive. Intriguingly, when we do not use the input graph, i.e., removing both the edge regularization and relational bias, \model can still yield competitive even superior results on \texttt{Mini-ImageNet}. This suggests that the $k$-NN graphs are not necessarily informative and besides, our model learns useful latent graph structures from data.

\subsection{Further Discussions}\label{sec-exp-further}

\textbf{Comparison of Time/Space Consumption.} Fig.~\ref{fig:timespace} plots training/inference time and GPU memory costs of \model and two SOTA structure learning models. Compared with LDS, \model reduces the training time, inference time, memory cost by up to 93.1\%, 97.9\%, 75.6\%, respectively; compared with IDGL (using anchor-based approximation for speedup), \model reduces the training time, inference time, memory cost by up to 61.8\%, 80.8\%, 80.6\%, respectively. 


\textbf{Ablation on Stochastic Components.} Table~\ref{tab:proteins} and \ref{tab:amazon} also include two variants of \model for ablation study. 1) \model-dt: replace Gumbel-Softmax by original Softmax (with temperature 1.0) for deterministic propagation; 2) \model-tp: use original Softmax with temperature set as 0.25 (the same as \model). There is performance drop when removing the Gumbel components, which may be due to over-normalizing or over-fitting that are amplified in large datasets, as we discussed in Section~\ref{sec-model-gnn} and the kernelized Gumbel-Softmax operator shows its effectiveness.

\textbf{Ablation on Edge Loss and Relational Bias.} 
We study the effects of edge-level regularization and relation bias as ablation study shown in Table~\ref{tab:ablation} located in Appendix~\ref{appx-result}, where the results consistently show that both components contribute to some positive effects and suggest that our edge-level loss and relation bias can both help to leverage useful information from input graphs.

\textbf{Impact of Temperature and Feature Map Dimension.} We study the effects of $\tau$ and $m$ in Fig.~\ref{fig:mt_ablation} located in Appendix~\ref{appx-result} and the variation trend accords with our theoretical analysis in Section~\ref{sec-model-theory}.
Specifically, the result shows that the test accuracy increases and then falls with the temperature changing from low to high values (usually achieves the peak accuracy with a temperature of 0.4). Besides, we can see that when the temperature is relatively small, the test accuracy goes high with the dimension of random features increasing. However, when the temperature is large, the accuracy would drop even with large feature dimension $m$.
Such a phenomenon accords with the theoretical result presented in Section~\ref{sec-model-theory}. For low temperature which enables desirable approximation performance for Gumbel-Softmax, then larger random feature dimension would help to produce better approximation to the original exponentiate-then-dot operator. In contrast, high temperature could not guarantee precise approximation for the original categorical distribution, which deteriorates the performance.

\textbf{Visualization and Implications.} 
Fig.~\ref{fig:vis-emb} visualizes node embeddings and edge connections (filter out the edges with weights larger than a threshold) on \texttt{20News-Groups} and \texttt{Mini-Imagenet}, which show that \model tends to assign more weights for nodes with the same class and sparse edges for nodes with different classes. This helps to interpret why \model improves the performance on downstream node-level prediction: the latent structures can propagate useful information to help the model learn better node representations that can be easily distinguished by the classifier.
We also compare the learned structures with original graphs in Fig.~\ref{fig:vis-att} located in Appendix~\ref{appx-result}. We can see that the latent structures learned by \model show different patterns from the observed ones, especially for heterophilic graphs. Another interesting phenomenon is that there exist some dominant nodes which are assigned large weights by other nodes, forming some vertical `lines' in the heatmap. This suggests that these nodes could contain critical information for the learning tasks and play as pivots that could improve the connectivity of the whole system.  


\begin{figure}[tb!]
    \centering
    \subfigure[\texttt{20News-Groups}]{
    \label{fig:20news}
    \includegraphics[width=5.5cm]{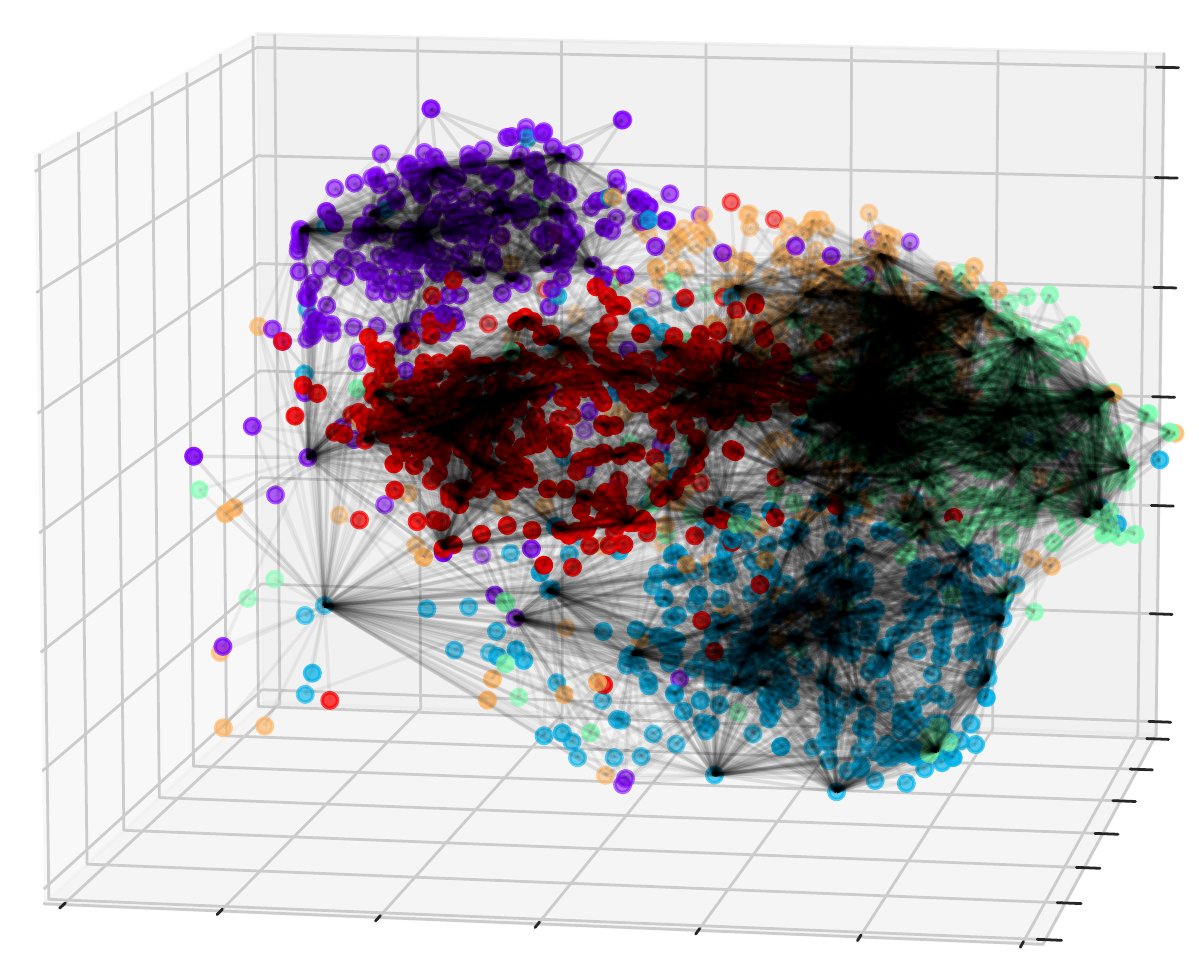}}
    \hspace{0.2in}
    \subfigure[\texttt{Mini-ImageNet}]{
    \label{fig:mini}
    \includegraphics[width=5.5cm]{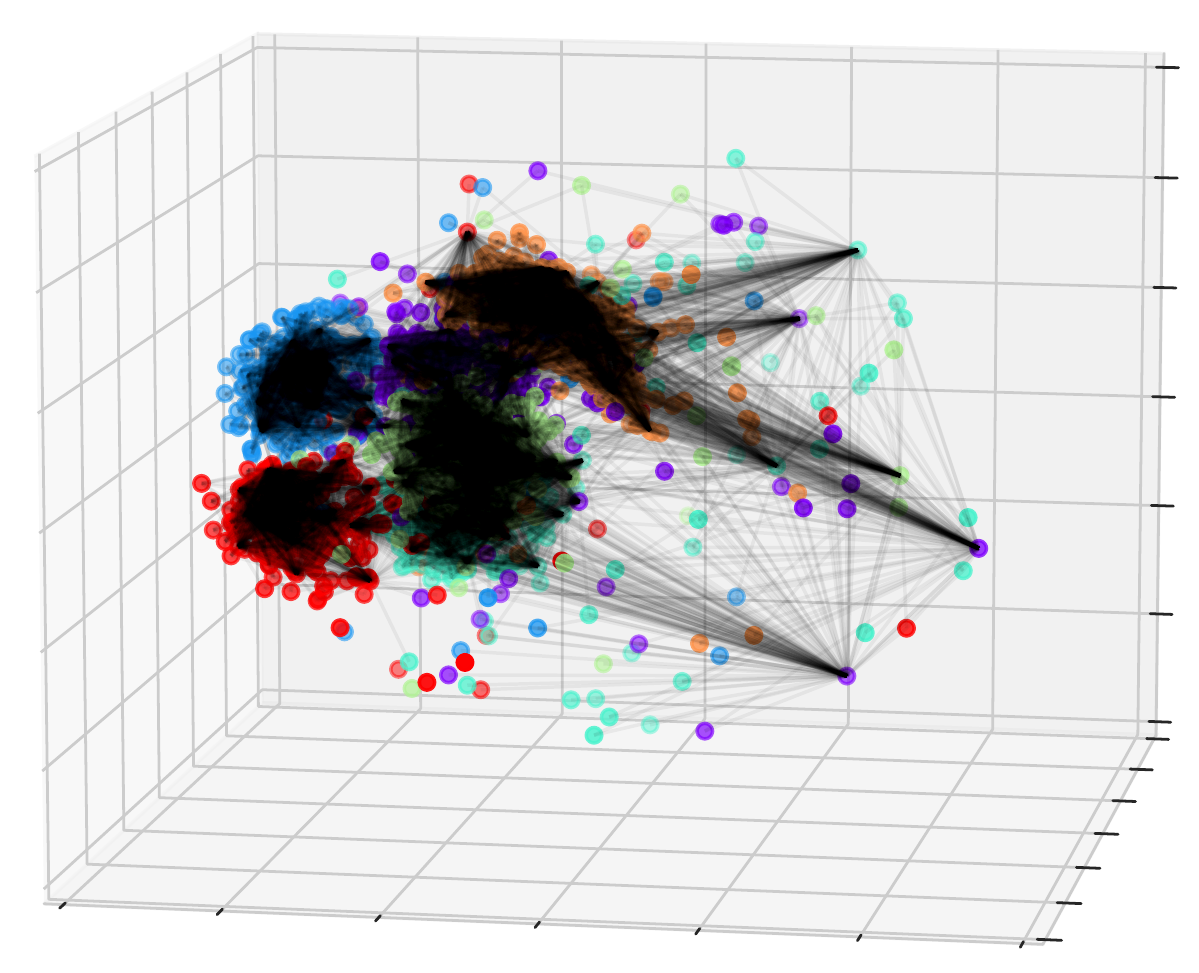}}
    \caption{Visualization of node embeddings and edge connections produced by \model on graph-enhanced application datasets. We mark the nodes with a particular class with one color. More comparison between the learned structures and original input graphs is presented in Appendix~\ref{appx-result}.}
    \label{fig:vis-emb}
    \vspace{-5pt}
\end{figure}

\section{Why \model Improves Downstream Prediction?}\label{sec-analysis}

There remains a natural question concerning our learning process: how effective can the learned latent topology be for downstream tasks? We next dissect the rationale from a Bayesian perspective. In fact, our model induces a predictive distribution $p(\mathbf Y, \tilde{\mathbf A}|\mathbf X, \mathbf A) = p(\tilde{\mathbf A}| \mathbf X, \mathbf A) p(\mathbf Y| \tilde{\mathbf A}, \mathbf X, \mathbf A)$ where we can treat the estimated graph $\tilde{\mathbf A}$ as a latent variable.\footnote{We assume one latent graph to simplify the illustration though we practically learn layer-specific graphs for each layer of \model. The analysis can be trivially extended to such a case.} Specifically, $p(\tilde{\mathbf A}| \mathbf X, \mathbf A)$ is instantiated with the structure estimation module and $p(\mathbf Y| \tilde{\mathbf A}, \mathbf X, \mathbf A)$ is instantiated with the feature propagation module. In principle, ideal latent graphs should account for downstream tasks and maximize the potentials of message passing for producing informative node representations. Thus, optimal 
latent graphs presumably come from the posterior $p(\tilde{\mathbf A}|\mathbf Y, \mathbf X, \mathbf A) = \frac{p(\mathbf Y|\mathbf X, \mathbf A, \tilde{\mathbf A}) p(\tilde{\mathbf A} | \mathbf X, \mathbf A)}{\int_{\mathbf Y} p(\mathbf Y|\mathbf X, \mathbf A, \tilde{\mathbf A}) p(\tilde{\mathbf A} | \mathbf X, \mathbf A) d\mathbf Y}$ which is given by Bayes theorem. Unfortunately, such a posterior is unknown and intractable for the integration. 

\textbf{A Variational Perspective.} An intriguing conclusion stems from another view into the learning process: we can treat the structure estimation as a variational distribution $q(\tilde{\mathbf A}|\mathbf X, \mathbf A)$ and our learning objective in Section~\ref{sec-training} can be viewed as the embodiment of a minimization problem over the predictive and variational distributions via
\begin{equation}\label{eqn-objective-elbo}
    p^*, q^* = \arg\min_{p, q}\underbrace{- \mathbb E_{q} [\log p(\mathbf Y| \tilde{\mathbf A}, \mathbf X, \mathbf A)]}_{\mathcal L_s} + \underbrace{\mathcal D(q(\tilde{\mathbf A}|\mathbf X, \mathbf A) \| p_0(\tilde{\mathbf A}|\mathbf X, \mathbf A)) }_{\mathcal L_e},
\end{equation}
where $\mathcal D$ denotes the Kullback-Leibler divergence. 
Specifically, the \emph{predictive} term is equivalent to minimizing the supervised loss (with Gumbel-Softmax as a surrogate for sampling-based estimates over $q(\tilde{\mathbf A}|\mathbf X, \mathbf A)$), and the KL \emph{regularization} term is embodied with the edge-level MLE loss (Eqn.~\ref{eqn-loss-mle}) (if we define the prior distribution $p_0(\tilde{\mathbf A}|\mathbf X, \mathbf A)$ following Eqn.~\ref{eqn-edge-prior}). One may notice that Eqn.~\ref{eqn-objective-elbo} is essentially the Evidence Lower Bound (ELBO) for the log-likelihood $\log p(\mathbf Y|\mathbf X, \mathbf A)$. 
\begin{proposition}\label{prop-var}
    Assume $q$ can exploit arbitrary distributions over $\tilde {\mathbf A}$. When Eqn.~\ref{eqn-objective-elbo} achieves the optimum, we have 1) $\mathcal D(q(\tilde{\mathbf A}|\mathbf X, \mathbf A) \| p(\tilde{\mathbf A}|\mathbf Y, \mathbf X, \mathbf A)) = 0$ and 2) $\log p(\mathbf Y| \mathbf X, \mathbf A)$ is maximized.
\end{proposition}
The proposition indicates that our adopted learning objective intrinsically minimizes the divergence between latent graphs generated by the model and the samples from the posterior $p(\tilde{\mathbf A}|\mathbf Y, \mathbf X, \mathbf A)$ that ideally helps to propagate useful adjacent information w.r.t. downstream tasks. Therefore, a well-trained network of \model on labeled data could produce effective latent topology that contributes to boosting the downstream performance.

\section{Conclusion}

This paper proposes a scalable and efficient graph Transformer (especially for node level) that can propagate layer-wise node signals between arbitrary pairs beyond input topology. The key module, a kernelized Gumbel-Softmax operator, enables us to learn layer-specific latent graphs with linear algorithmic complexity without compromising the precision. The results on diverse graph datasets and situations verify the effectiveness, scalability, and stability. We provide more discussions on the limitations and potential impacts in Appendix~\ref{appx-impact}.

\section*{Acknowledgement}\vspace{-5pt}
This work was partly supported by National Key Research and Development Program of China (2020AAA0107600), National Natural Science Foundation of China (61972250, 72061127003), and Shanghai Municipal Science and Technology (Major) Project (22511105100, 2021SHZDZX0102).

{
\bibliographystyle{plain}
\bibliography{ref}

\begin{thebibliography}{10}

\bibitem{mixhop-icml19}
Sami Abu{-}El{-}Haija, Bryan Perozzi, Amol Kapoor, Nazanin Alipourfard,
  Kristina Lerman, Hrayr Harutyunyan, Greg~Ver Steeg, and Aram Galstyan.
\newblock Mixhop: Higher-order graph convolutional architectures via sparsified
  neighborhood mixing.
\newblock In {\em International Conference on Machine Learning}, pages 21--29,
  2019.

\bibitem{oversquashing-iclr21}
Uri Alon and Eran Yahav.
\newblock On the bottleneck of graph neural networks and its practical
  implications.
\newblock In {\em International Conference on Learning Representations}, 2021.

\bibitem{geometriclearning-2016}
Michael~M. Bronstein, Joan Bruna, Yann LeCun, Arthur Szlam, and Pierre
  Vandergheynst.
\newblock Geometric deep learning: going beyond euclidean data.
\newblock {\em CoRR}, abs/1611.08097, 2016.

\bibitem{IDGL-neurips20}
Yu~Chen, Lingfei Wu, and Mohammed~J. Zaki.
\newblock Iterative deep graph learning for graph neural networks: Better and
  robust node embeddings.
\newblock In {\em Advances in Neural Information Processing Systems}, 2020.

\bibitem{clustergcn-kdd19}
Wei{-}Lin Chiang, Xuanqing Liu, Si~Si, Yang Li, Samy Bengio, and Cho{-}Jui
  Hsieh.
\newblock Cluster-gcn: An efficient algorithm for training deep and large graph
  convolutional networks.
\newblock In {\em {ACM} {SIGKDD} International Conference on Knowledge
  Discovery {\&} Data Mining}, pages 257--266, 2019.

\bibitem{performer-iclr21}
Krzysztof~Marcin Choromanski, Valerii Likhosherstov, David Dohan, Xingyou Song,
  Andreea Gane, Tam{\'{a}}s Sarl{\'{o}}s, Peter Hawkins, Jared~Quincy Davis,
  Afroz Mohiuddin, Lukasz Kaiser, David~Benjamin Belanger, Lucy~J. Colwell, and
  Adrian Weller.
\newblock Rethinking attention with performers.
\newblock In {\em International Conference on Learning Representations}, 2021.

\bibitem{patient-2020}
Luca Cosmo, Anees Kazi, Seyed{-}Ahmad Ahmadi, Nassir Navab, and Michael~M.
  Bronstein.
\newblock Latent patient network learning for automatic diagnosis.
\newblock {\em CoRR}, abs/2003.13620, 2020.

\bibitem{longrange-icml18}
Hanjun Dai, Zornitsa Kozareva, Bo~Dai, Alexander~J. Smola, and Le~Song.
\newblock Learning steady-states of iterative algorithms over graphs.
\newblock In {\em International Conference on Machine Learning}, pages
  1114--1122, 2018.

\bibitem{graphtransformer-2020}
Vijay~Prakash Dwivedi and Xavier Bresson.
\newblock A generalization of transformer networks to graphs.
\newblock {\em CoRR}, abs/2012.09699, 2020.

\bibitem{varstruct-neurips20}
Pantelis Elinas, Edwin~V. Bonilla, and Louis~C. Tiao.
\newblock Variational inference for graph convolutional networks in the absence
  of graph data and adversarial settings.
\newblock In {\em Advances in Neural Information Processing Systems}, 2020.

\bibitem{IDLS-icml19}
Luca Franceschi, Mathias Niepert, Massimiliano Pontil, and Xiao He.
\newblock Learning discrete structures for graph neural networks.
\newblock In {\em International Conference on Machine Learning}, pages
  1972--1982, 2019.

\bibitem{image-cvpr21}
Chen Gao, Jinyu Chen, Si~Liu, Luting Wang, Qiong Zhang, and Qi~Wu.
\newblock Room-and-object aware knowledge reasoning for remote embodied
  referring expression.
\newblock In {\em {IEEE} Conference on Computer Vision and Pattern
  Recognition}, pages 3064--3073, 2021.

\bibitem{longrange-nips20}
Fangda Gu, Heng Chang, Wenwu Zhu, Somayeh Sojoudi, and Laurent~El Ghaoui.
\newblock Implicit graph neural networks.
\newblock In {\em Advances in Neural Information Processing Systems}, 2020.

\bibitem{graphsage}
William~L. Hamilton, Zhitao Ying, and Jure Leskovec.
\newblock Inductive representation learning on large graphs.
\newblock In {\em Advances in Neural Information Processing Systems}, pages
  1024--1034, 2017.

\bibitem{ogb-nips20}
Weihua Hu, Matthias Fey, Marinka Zitnik, Yuxiao Dong, Hongyu Ren, Bowen Liu,
  Michele Catasta, and Jure Leskovec.
\newblock Open graph benchmark: Datasets for machine learning on graphs.
\newblock In {\em Advances in Neural Information Processing Systems}, 2020.

\bibitem{gumbel-iclr17}
Eric Jang, Shixiang Gu, and Ben Poole.
\newblock Categorical reparameterization with gumbel-softmax.
\newblock In {\em International Conference on Learning Representations}, 2017.

\bibitem{glconv-cvpr19}
Bo~Jiang, Ziyan Zhang, Doudou Lin, Jin Tang, and Bin Luo.
\newblock Semi-supervised learning with graph learning-convolutional networks.
\newblock In {\em {IEEE} Conference on Computer Vision and Pattern
  Recognition}, pages 11313--11320, 2019.

\bibitem{structnorm-kdd20}
Wei Jin, Yao Ma, Xiaorui Liu, Xianfeng Tang, Suhang Wang, and Jiliang Tang.
\newblock Graph structure learning for robust graph neural networks.
\newblock In {\em {ACM} {SIGKDD} Conference on Knowledge Discovery and Data
  Mining}, pages 66--74, 2020.

\bibitem{GCN-vallina}
Thomas~N. Kipf and Max Welling.
\newblock Semi-supervised classification with graph convolutional networks.
\newblock In {\em International Conference on Learning Representations (ICLR)},
  2017.

\bibitem{wsgnn}
Danning Lao, Xinyu Yang, Qitian Wu, and Junchi Yan.
\newblock Variational inference for training graph neural networks in low-data
  regime through joint structure-label estimation.
\newblock In {\em {ACM} {SIGKDD} Conference on Knowledge Discovery and Data
  Mining}, pages 824--834, 2022.

\bibitem{newbench}
Derek Lim, Xiuyu Li, Felix Hohne, and Ser{-}Nam Lim.
\newblock New benchmarks for learning on non-homophilous graphs.
\newblock {\em CoRR}, abs/2104.01404, 2021.

\bibitem{learn2drop-wsdm20}
Dongsheng Luo, Wei Cheng, Wenchao Yu, Bo~Zong, Jingchao Ni, Haifeng Chen, and
  Xiang Zhang.
\newblock Learning to drop: Robust graph neural network via topological
  denoising.
\newblock In {\em {ACM} International Conference on Web Search and Data
  Mining}, pages 779--787, 2021.

\bibitem{Concrete-iclr17}
Chris~J. Maddison, Andriy Mnih, and Yee~Whye Teh.
\newblock The concrete distribution: {A} continuous relaxation of discrete
  random variables.
\newblock In {\em International Conference on Learning Representations}, 2017.

\bibitem{amazoncopurchase-kdd15}
Julian~J. McAuley, Rahul Pandey, and Jure Leskovec.
\newblock Inferring networks of substitutable and complementary products.
\newblock In {\em {ACM} {SIGKDD} International Conference on Knowledge
  Discovery and Data Mining}, pages 785--794, 2015.

\bibitem{pedregosa2011scikit}
Fabian Pedregosa, Ga{\"e}l Varoquaux, Alexandre Gramfort, Vincent Michel,
  Bertrand Thirion, Olivier Grisel, Mathieu Blondel, Peter Prettenhofer, Ron
  Weiss, Vincent Dubourg, et~al.
\newblock Scikit-learn: Machine learning in python.
\newblock {\em the Journal of machine Learning research}, 12:2825--2830, 2011.

\bibitem{geomgcn-iclr20}
Hongbin Pei, Bingzhe Wei, Kevin~Chen{-}Chuan Chang, Yu~Lei, and Bo~Yang.
\newblock Geom-gcn: Geometric graph convolutional networks.
\newblock In {\em International Conference on Learning Representations}, 2020.

\bibitem{rff}
Ali Rahimi and Benjamin Recht.
\newblock Random features for large-scale kernel machines.
\newblock In {\em Advances in Neural Information Processing Systems}, pages
  1177--1184, 2007.

\bibitem{dropedge-iclr20}
Yu~Rong, Wenbing Huang, Tingyang Xu, and Junzhou Huang.
\newblock Dropedge: Towards deep graph convolutional networks on node
  classification.
\newblock In {\em International Conference on Learning Representations}, 2020.

\bibitem{deezer-cikm19}
Benedek Rozemberczki and Rik Sarkar.
\newblock Characteristic functions on graphs: Birds of a feather, from
  statistical descriptors to parametric models.
\newblock In {\em {ACM} International Conference on Information and Knowledge
  Management}, pages 1325--1334, 2020.

\bibitem{physics-icml20}
Alvaro Sanchez{-}Gonzalez, Jonathan Godwin, Tobias Pfaff, Rex Ying, Jure
  Leskovec, and Peter~W. Battaglia.
\newblock Learning to simulate complex physics with graph networks.
\newblock In {\em International Conference on Machine Learning}, pages
  8459--8468, 2020.

\bibitem{garcia2018fewshot}
Victor~Garcia Satorras and Joan~Bruna Estrach.
\newblock Few-shot learning with graph neural networks.
\newblock In {\em International Conference on Learning Representations}, 2018.

\bibitem{scarselli2008gnnearly}
Franco Scarselli, Marco Gori, Ah~Chung Tsoi, Markus Hagenbuchner, and Gabriele
  Monfardini.
\newblock The graph neural network model.
\newblock {\em IEEE transactions on neural networks}, 20(1):61--80, 2008.

\bibitem{Sen08collectiveclassification}
Prithviraj Sen, Galileo Namata, Mustafa Bilgic, Lise Getoor, Brian Gallagher,
  and Tina Eliassi{-}Rad.
\newblock Collective classification in network data.
\newblock {\em {AI} Mag.}, 29(3):93--106, 2008.

\bibitem{fastgat-20}
Rakshith~Sharma Srinivasa, Cao Xiao, Lucas Glass, Justin Romberg, and Jimeng
  Sun.
\newblock Fast graph attention networks using effective resistance based graph
  sparsification.
\newblock {\em CoRR}, abs/2006.08796, 2020.

\bibitem{vaswani2017attention}
Ashish Vaswani, Noam Shazeer, Niki Parmar, Jakob Uszkoreit, Llion Jones,
  Aidan~N Gomez, {\L}ukasz Kaiser, and Illia Polosukhin.
\newblock Attention is all you need.
\newblock {\em Advances in neural information processing systems}, 30, 2017.

\bibitem{GAT}
Petar Velickovic, Guillem Cucurull, Arantxa Casanova, Adriana Romero, Pietro
  Li{\`{o}}, and Yoshua Bengio.
\newblock Graph attention networks.
\newblock In {\em International Conference on Learning Representations (ICLR)},
  2018.

\bibitem{miniimagenet-neurips2016}
Oriol Vinyals, Charles Blundell, Tim Lillicrap, Koray Kavukcuoglu, and Daan
  Wierstra.
\newblock Matching networks for one shot learning.
\newblock In {\em Advances in Neural Information Processing Systems}, pages
  3630--3638, 2016.

\bibitem{pointcloud-19}
Yue Wang, Yongbin Sun, Ziwei Liu, Sanjay~E. Sarma, Michael~M. Bronstein, and
  Justin~M. Solomon.
\newblock Dynamic graph {CNN} for learning on point clouds.
\newblock {\em {ACM} Trans. Graph.}, 38(5):146:1--146:12, 2019.

\bibitem{SGC-icml19}
Felix Wu, Amauri H.~Souza Jr., Tianyi Zhang, Christopher Fifty, Tao Yu, and
  Kilian~Q. Weinberger.
\newblock Simplifying graph convolutional networks.
\newblock In {\em International Conference on Machine Learning}, pages
  6861--6871, 2019.

\bibitem{wu2021feature}
Qitian Wu, Chenxiao Yang, and Junchi Yan.
\newblock Towards open-world feature extrapolation: An inductive graph learning
  approach.
\newblock {\em Advances in Neural Information Processing Systems}, pages
  19435--19447, 2021.

\bibitem{wu2021rec}
Qitian Wu, Hengrui Zhang, Xiaofeng Gao, Junchi Yan, and Hongyuan Zha.
\newblock Towards open-world recommendation: An inductive model-based
  collaborative filtering approach.
\newblock In {\em International Conference on Machine Learning}, pages
  11329--11339, 2021.

\bibitem{GIB-neurips20}
Tailin Wu, Hongyu Ren, Pan Li, and Jure Leskovec.
\newblock Graph information bottleneck.
\newblock In {\em Advances in Neural Information Processing Systems}, 2020.

\bibitem{kernalstruct-cikm18}
Xuan Wu, Lingxiao Zhao, and Leman Akoglu.
\newblock A quest for structure: Jointly learning the graph structure and
  semi-supervised classification.
\newblock In {\em {ACM} International Conference on Information and Knowledge
  Management}, pages 87--96, 2018.

\bibitem{jknet-icml18}
Keyulu Xu, Chengtao Li, Yonglong Tian, Tomohiro Sonobe, Ken{-}ichi
  Kawarabayashi, and Stefanie Jegelka.
\newblock Representation learning on graphs with jumping knowledge networks.
\newblock In {\em International Conference on Machine Learning}, pages
  5449--5458, 2018.

\bibitem{geokd}
Chenxiao Yang, Qitian Wu, and Junchi Yan.
\newblock Geometric knowledge distillation: Topology compression for graph
  neural networks.
\newblock In {\em Advances in Neural Information Processing Systems}, 2022.

\bibitem{text-aaai20}
Liang Yao, Chengsheng Mao, and Yuan Luo.
\newblock Graph convolutional networks for text classification.
\newblock In {\em {AAAI} Conference on Artificial Intelligence}, pages
  7370--7377, 2019.

\bibitem{gnnexplain-neurips19}
Rex Ying, Dylan Bourgeois, Jiaxuan You, Marinka Zitnik, and Jure Leskovec.
\newblock {GNN} explainer: {A} tool for post-hoc explanation of graph neural
  networks.
\newblock In {\em Advances in Neural Information Processing Systems}, 2019.

\bibitem{zeng2019graphsaint}
Hanqing Zeng, Hongkuan Zhou, Ajitesh Srivastava, Rajgopal Kannan, and Viktor
  Prasanna.
\newblock Graphsaint: Graph sampling based inductive learning method.
\newblock In {\em International Conference on Learning Representations}, 2020.

\bibitem{zhang2022scalegcn}
Tianqi Zhang, Qitian Wu, Junchi Yan, Yunan Zhao, and Bing Han.
\newblock Scalegcn: Efficient and effective graph convolution via channel-wise
  scale transformation.
\newblock {\em IEEE Transactions on Neural Networks and Learning Systems},
  2022.

\bibitem{gnnguard-neurips20}
Xiang Zhang and Marinka Zitnik.
\newblock Gnnguard: Defending graph neural networks against adversarial
  attacks.
\newblock In {\em Advances in Neural Information Processing Systems}, 2020.

\bibitem{Bayesstruct-aaai19}
Yingxue Zhang, Soumyasundar Pal, Mark Coates, and Deniz {\"{U}}stebay.
\newblock Bayesian graph convolutional neural networks for semi-supervised
  classification.
\newblock In {\em {AAAI} Conference on Artificial Intelligence}, pages
  5829--5836, 2019.

\bibitem{neuralsparse-icml20}
Cheng Zheng, Bo~Zong, Wei Cheng, Dongjin Song, Jingchao Ni, Wenchao Yu, Haifeng
  Chen, and Wei Wang.
\newblock Robust graph representation learning via neural sparsification.
\newblock In {\em International Conference on Machine Learning}, pages
  11458--11468, 2020.

\bibitem{h2gcn-neurips20}
Jiong Zhu, Yujun Yan, Lingxiao Zhao, Mark Heimann, Leman Akoglu, and Danai
  Koutra.
\newblock Beyond homophily in graph neural networks: Current limitations and
  effective designs.
\newblock In {\em Advances in Neural Information Processing Systems}, 2020.

\bibitem{GSL-survey}
Yanqiao Zhu, Weizhi Xu, Jinghao Zhang, Qiang Liu, Shu Wu, and Liang Wang.
\newblock Deep graph structure learning for robust representations: {A} survey.
\newblock {\em CoRR}, abs/2103.03036, 2021.

\end{thebibliography}
}
\newpage

\section*{Checklist}


\begin{enumerate}

\item For all authors...
\begin{enumerate}
  \item Do the main claims made in the abstract and introduction accurately reflect the paper's contributions and scope?
    \answerYes{}
  \item Did you describe the limitations of your work?
    \answerYes{}
  \item Did you discuss any potential negative societal impacts of your work?
    \answerYes{}
  \item Have you read the ethics review guidelines and ensured that your paper conforms to them?
    \answerYes{}
\end{enumerate}

\item If you are including theoretical results...
\begin{enumerate}
  \item Did you state the full set of assumptions of all theoretical results?
    \answerYes{}
        \item Did you include complete proofs of all theoretical results?
    \answerYes{See Appendix~\ref{appx-proof}}
\end{enumerate}

\item If you ran experiments...
\begin{enumerate}
  \item Did you include the code, data, and instructions needed to reproduce the main experimental results (either in the supplemental material or as a URL)?
    \answerYes{The codes are public available. See Appendix~\ref{appx-dataset} for dataset information.}
  \item Did you specify all the training details (e.g., data splits, hyperparameters, how they were chosen)?
    \answerYes{See Appendix~\ref{appx-implementation}}
        \item Did you report error bars (e.g., with respect to the random seed after running experiments multiple times)?
    \answerYes{See the experiment section}
        \item Did you include the total amount of compute and the type of resources used (e.g., type of GPUs, internal cluster, or cloud provider)?
    \answerYes{See Appendix~\ref{appx-implementation}}
\end{enumerate}

\item If you are using existing assets (e.g., code, data, models) or curating/releasing new assets...
\begin{enumerate}
  \item If your work uses existing assets, did you cite the creators?
    \answerYes{See Appendix~\ref{appx-dataset}}
    \item Did you mention the license of the assets?
    \answerNA{}
  \item Did you include any new assets either in the supplemental material or as a URL?
    \answerNA{}
  \item Did you discuss whether and how consent was obtained from people whose data you're using/curating?
    \answerNA{}
  \item Did you discuss whether the data you are using/curating contains personally identifiable information or offensive content?
    \answerNA{}
\end{enumerate}

\item If you used crowdsourcing or conducted research with human subjects...
\begin{enumerate}
  \item Did you include the full text of instructions given to participants and screenshots, if applicable?
    \answerNA{}
  \item Did you describe any potential participant risks, with links to Institutional Review Board (IRB) approvals, if applicable?
    \answerNA{}
  \item Did you include the estimated hourly wage paid to participants and the total amount spent on participant compensation?
    \answerNA{}
\end{enumerate}

\end{enumerate}


\clearpage
\appendix
\section*{Appendix}

\section{More Details for \model}\label{appx-model}

\subsection{Differentiable Sampling-based Message Passing on Latent Structures}

We provide more details concerning the differentiable sampling-based message passing through our kernelized Gumbel-Softmax operator, as complementary to the content of Sec.~\ref{sec-model-gnn}. As illustrated in Sec.~\ref{sec-model-gnn}, the $l$-th layer's feature propagation is defined over the $l$-th layer's latent graph composed of the sampled edges $e_{uv}^{(l)}\sim \mbox{Cat}(\bm \pi_u)^{(l)}$. For each layer, we sample $K$ times for each node, i.e., there will be $K$ sampled neighbored nodes for each node $u$. We assume $\tilde{\mathcal E}^{(l)} = \{e_{uv}^{(l)}\}$ as the set of sampled edges in the latent graph of the $l$-th layer. Then the updating rule for node embeddings at the $l$-th layer based on the latent graph can be written as
\begin{equation}
    \mathbf z_u^{(l+1)} = \frac{1}{K} \sum_{v, e_{uv}^{(l)}\in \tilde{\mathcal E}^{(l)}} \mathbf v_u = \frac{1}{K} \sum_{v} \mathbb I[e_{uv}^{(l)}\in \tilde{\mathcal E}^{(l)}] \mathbf v_u.
\end{equation}
The above equation introduces dis-continuity due to the sampling process that disables the end-to-end differentiable training. We thus adopt Gumbel-Softmax as a reparameterization trick to approximate the discrete sampled results via continuous relaxation:
\begin{equation}\label{eqn-appx-attn-gumbel-ori}
    \mathbf z_u^{(l+1)} \approx \frac{1}{K}\sum_{k=1}^K \sum_{v=1}^N \frac{\exp((\mathbf q_u^\top \mathbf k_u + g_{kv}) /\tau)}{\sum_{w=1}^N \exp((\mathbf q_u^\top \mathbf k_w + g_{kw})/\tau)} \cdot \mathbf v_u, ~ g_{kw}\sim \mbox{Gumbel}(0, 1).
\end{equation}
The temperature $\tau$ controls the closeness to hard discrete samples~\cite{Concrete-iclr17}. If $\tau$ is close to zero, then the Gumbel-Softmax term $\frac{\exp((\mathbf q_u^\top \mathbf k_u + g_{kv}) /\tau)}{\sum_{w=1}^N \exp((\mathbf q_u^\top \mathbf k_w + g_{kw})/\tau)}$ for any $v$ converges to a one-hot vector:
\begin{equation}
\begin{aligned}
    \frac{\exp((\mathbf q_u^\top \mathbf k_v + g_{kv}) /\tau)}{\sum_{w=1}^N \exp((\mathbf q_u^\top \mathbf k_w + g_{kw})/\tau)}=&\left\{ 
    \begin{array}{ll}
         & 1, \quad \mbox{if} \; v~\mbox{satisfies}~ \mathbf q_u^\top \mathbf k_v + g_{kv} > \mathbf q_u^\top \mathbf k_{v'} + g_{kv'} \forall v' \neq v,  \\
         & 0, \quad otherwise.
    \end{array}
    \right. 
    \end{aligned}
\end{equation}
The Eqn.~\ref{eqn-appx-attn-gumbel-ori} requires $\mathcal O(N^2)$ for computing the embeddings for $N$ nodes in one layer. To reduce the complexity to $\mathcal O(N)$, we resort to the kernel approximation idea, following similar reasoning as Eqn.~\ref{eqn-attn-kernel} and \ref{eqn-attn-rff}:
\begin{equation}
\begin{aligned}
\mathbf z_u^{(l+1)} &\approx \frac{1}{K}\sum_{k=1}^K \sum_{v=1}^N \frac{\exp((\mathbf q_u^\top \mathbf k_u + g_{kv}) /\tau)}{\sum_{w=1}^N \exp((\mathbf q_u^\top \mathbf k_w + g_{kw})/\tau)} \cdot \mathbf v_u \\
    & = \frac{1}{K}\sum_{k=1}^K \sum_{v=1}^N \frac{\exp((\mathbf q_u^\top \mathbf k_u + g_{kv}) /\tau)}{\sum_{w=1}^N \exp((\mathbf q_u^\top \mathbf k_w + g_{kw})/\tau)} \cdot \mathbf v_u \\
    & = \frac{1}{K}\sum_{k=1}^K \sum_{v=1}^N \frac{\kappa(\mathbf q_u / \sqrt{\tau}, \mathbf k_v / \sqrt{\tau}) e^{g_{kv}/\tau}}{\sum_{w=1}^N \kappa(\mathbf q_u / \sqrt{\tau}, \mathbf k_w / \sqrt{\tau}) e^{g_{kw}/\tau}} \cdot \mathbf v_v \\
    & \approx \frac{1}{K}\sum_{k=1}^K \sum_{v=1}^N \frac{\phi(\mathbf q_u / \sqrt{\tau})^\top \phi(\mathbf k_v / \sqrt{\tau}) e^{g_{kv}/\tau}}{\sum_{w=1}^N \phi(\mathbf q_u / \sqrt{\tau})^\top \phi(\mathbf k_w / \sqrt{\tau}) e^{g_{kw}/\tau}} \cdot \mathbf v_v  \\
    & = \frac{1}{K}\sum_{k=1}^K \frac{\phi(\mathbf q_u / \sqrt{\tau})^\top  \sum_{v=1}^N e^{g_{kv}/\tau}\phi(\mathbf k_v / \sqrt{\tau}) \cdot \mathbf v_v^\top  }{ \phi(\mathbf q_u / \sqrt{\tau})^\top \sum_{w=1}^N e^{g_{kw}/\tau}\phi(\mathbf k_w / \sqrt{\tau})}. 
    \end{aligned}
\end{equation}
The above result yields the one-layer updating rule for \model's feed-forwarding w.r.t. each node $u$. In terms of practical implementation, we adopt matrix multiplications for computing the node embeddings for all the nodes in the next layer, for which we present the details in the next subsection.

\subsection{Model Implementation from the Matrix View}

In practice, the implementation of \model is based on matrix operations that simultanenously update all the nodes in one layer. We present the feed-forward process of \model from a matrix view in Fig.~\ref{fig-model} where Alg.~1 depicts how node embeddings are updated in each layer through our introduced kernelized Gumbel-Softmax message passing in Sec.~\ref{sec-model-gnn}. The right sub-figure illustrates the one layer's updating which only requires $\mathcal O(N)$ complexity by avoiding the cumbersome all-pair similarity matrix. 

\begin{figure*}[t!]
\centering
\begin{minipage}[t]{0.49\linewidth}
\centering
\label{fig-alg}
\includegraphics[width=\textwidth]{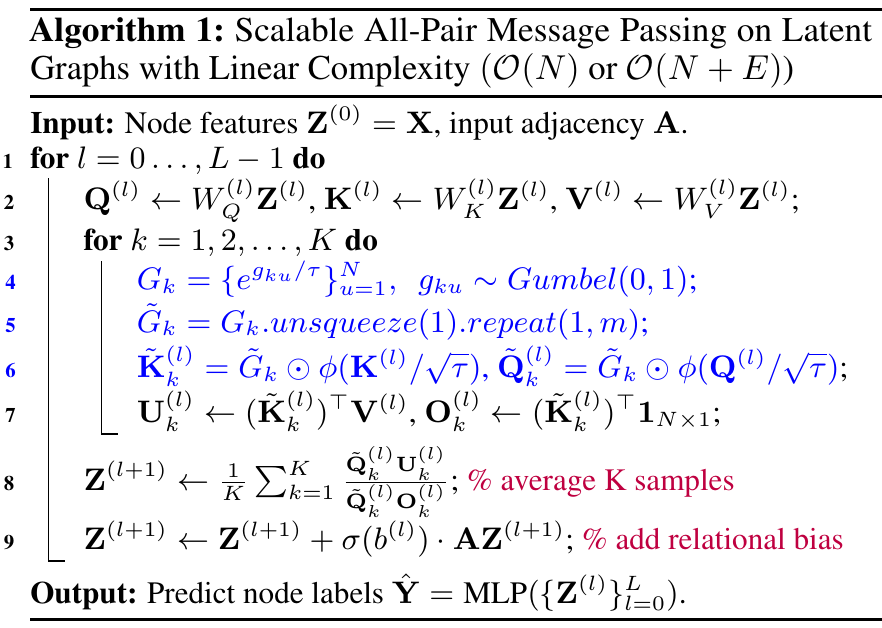}
\end{minipage}%
\hspace{2pt}
\begin{minipage}[t]{0.49\linewidth}
\centering
\label{fig-model}
\includegraphics[width=0.98\textwidth]{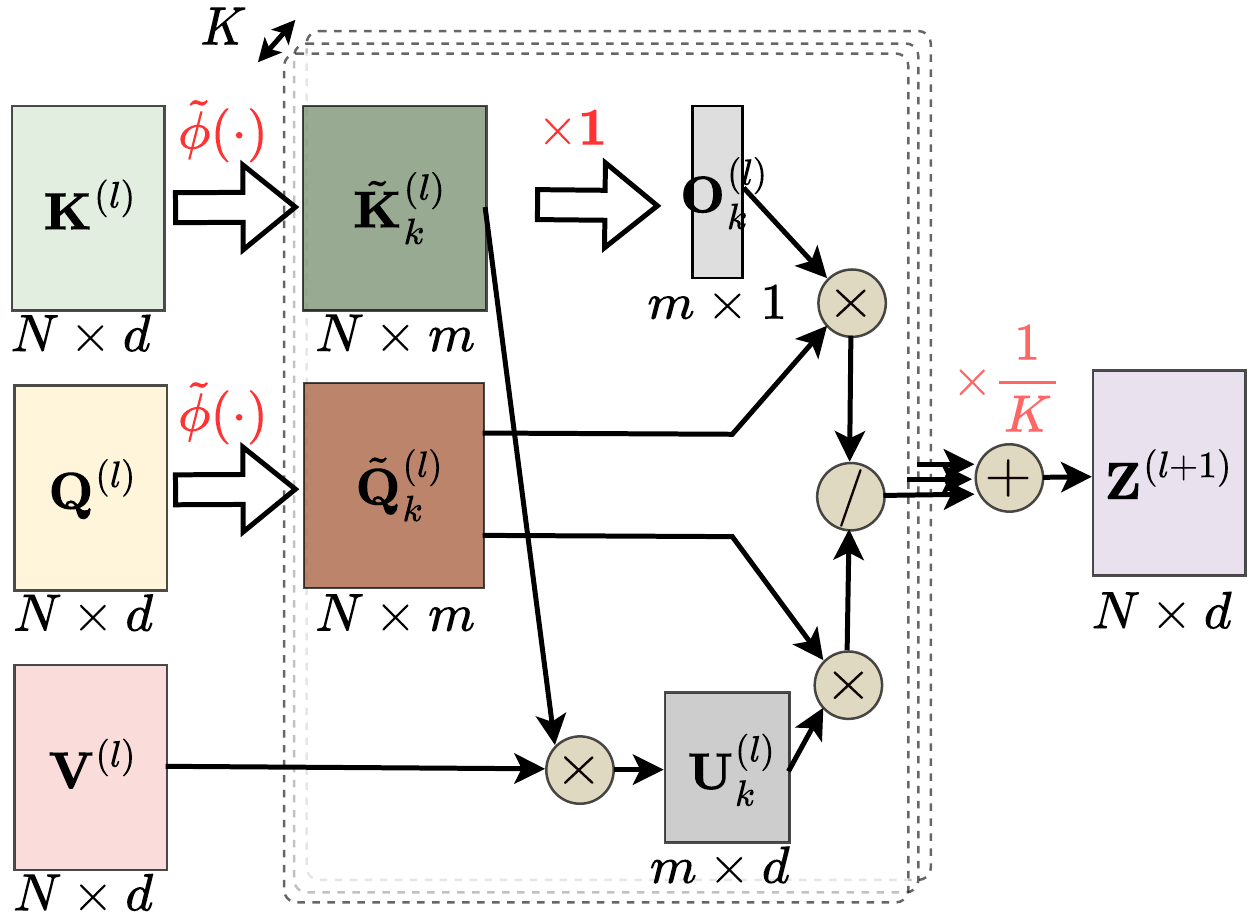}
\end{minipage}%
\caption{Alg.~1 presents the details for \model's feed-forward process from a matrix view that is practically used in our implementation. The figure illustrates the layer-wise node representation updating based on the kernelized Gumbel-Softmax operator, which reduces the algorithmic complexity from quadratic to $\mathcal O(N)$ via avoiding explicit computation of the all-pair similarities. $\odot$ in Alg.~1 denotes element-wise product. $\tilde \phi(\cdot)$ in the figure represents the random feature map with Gumbel noise whose details are shown by the blue part of Alg.~1.}
\label{fig-model}
\end{figure*}

\section{Proof for Technical Results}\label{appx-proof}

\subsection{Proof for Theorem~\ref{thm-error}}

To prove our theorem, we first introduce the following lemma given by the Lemma 2 in~\cite{performer-iclr21}.
\begin{lemma}
\label{lem:kernel}
Denote a softmax kernel as $\mbox{SM}(\mathbf x, \mathbf y) = \exp(\mathbf x^\top \mathbf y)$. The Positive Random Features defined by Eqn.~\ref{eqn-pos-rff} for softmax-kernel estimation, i.e., $\widehat{\mbox{SM}}_m(\mathbf x,\mathbf y) = \frac{1}{m}\sum_{i=1}^m[\exp (\mathbf w_i^\top \mathbf x - \frac{\|\mathbf x\|^2}{2}) \exp(\mathbf w_i^\top \mathbf y - \frac{\|\mathbf y\|^2}{2})]$,
has the mean and variance over $\mathbf w\sim \mathcal N(0, I_d)$ as
\begin{equation}
    \begin{aligned}
    \mathbb{E}_{\mathbf w}(\widehat{\mbox{SM}}_m(\mathbf x,\mathbf y)) = &\mbox{SM}(\mathbf x, \mathbf y) = \exp(\mathbf x^\top \mathbf y), \\
    \mathbb{V}_{\mathbf w}(\widehat{\mbox{SM}}_m(\mathbf x,\mathbf y)) = &\frac{1}{m}\exp(\|\mathbf x+ \mathbf y\|^2)\mbox{SM}^2(\mathbf x,\mathbf y) \\
    &(1-\exp(-\|\mathbf x+ \mathbf y\|^2)).
    \end{aligned}
\end{equation}
\end{lemma}
The lemma shows that the Positive Random Features can achieve unbiased approximation for the softmax kernel with a quantified variance. 

Back to our main theorem, suppose the L2-norms of $\mathbf q_u$ and $\mathbf k_v$ are bounded by $r$, we can derive the probability using the Chebyshev's inequality:
\begin{equation}
    \label{eq:cheby}
    \begin{aligned}
    \mathbb P (\Delta \leq \sqrt{\frac{\exp(6r/\tau)}{m\epsilon}}) \geq 1-\frac{\mathbb{V}_{\mathbf w}(\widehat{\mbox{SM}}_m(\mathbf q_u/\sqrt{\tau},\mathbf k_v/\sqrt{\tau}))}{\exp(6r/\tau)/m\epsilon}
    \end{aligned}
\end{equation}
where $\Delta = \left | \widehat{\mbox{SM}}_m(\mathbf q_u/\sqrt{\tau}, \mathbf k_v/\sqrt{\tau}) - \mbox{SM}(\mathbf q_u/\sqrt{\tau}, \mathbf k_v/\sqrt{\tau}) \right |$ denotes the deviation of the kernel approximation. Using the result in Lemma~\ref{lem:kernel}, we can further obtain that the RHS of Eqn.~\ref{eq:cheby} is no greater than
\begin{equation}
    1- \epsilon \exp(\|\frac{\mathbf q_u + \mathbf k_v}{\sqrt{\tau}}\|^2+2\frac{\mathbf q_u^\top \mathbf k_v}{\tau} -6\frac{r}{\tau}).
\end{equation}
Since $\|\frac{\mathbf q_u + \mathbf k_v}{\sqrt{\tau}}\|^2 \leq \frac{4r}{\tau}$ and $2\frac{\mathbf q_u^\top \mathbf k_v}{\tau} \leq \frac{2r}{\tau}$, we can achieve the stated result:
\begin{equation}
    \mathbb P(\Delta \leq \sqrt{\frac{\exp(6r/\tau)}{m\epsilon}}) \geq 1-\epsilon.
\end{equation}

\subsection{Proof for Theorem~\ref{thm-gumbel}}

Before entering the proof for the theorem, we first introduce two basic technical lemmas. While such results are already mentioned in previous studies~\cite{gumbel-iclr17, Concrete-iclr17}, their proofs will be useful for the subsequent reasoning. Therefore, we restate the proofs as building blocks for the following presentation.

\label{sec:proofs}
\begin{lemma}
\label{lem:gumbel}
Given real numbers $x_i, x_j\in \mathbb R$ and $u_i, u_j$ i.i.d. sampled from uniform distribution within $(0, 1)$. With Gumbel perturbation defined as $g(u)=-\log(-\log(u))$, we have the probability
$$
P(x_i+g(u_i) > x_j + g(u_j)) = \frac{1}{1+\exp{(-(x_i-x_j))}}. 
$$
\begin{proof}
Due to $g(u)=-\log(-\log(u))$, the inequality of interests $x_i+g(u_i) > x_j + g(u_j)$ can be rearranged as
\begin{equation}\label{eqn-proof-1-1}
    e^{x_i-x_j} > \frac{\log(u_i)}{\log(u_j)}.
\end{equation}
Since $\log(u_j) < 0$, Eqn.~\ref{eqn-proof-1-1} can be written as
\begin{equation}
    u_j < u_i^{e^{x_j-x_i}}.
\end{equation}
As $u_i,u_j$ are i.i.d. sampled from a uniform distribution, the probability when the above formula can be calculated via:
\begin{equation}
\begin{aligned}
    \int_0^1\int_0^{u_i^{e^{x_j-x_i}}}du_jdu_i =& \int_0^1u_i^{e^{x_j-x_i}}du_i\\
    =&\frac{1}{1+\exp(-(x_i-x_j))}.
\end{aligned}
\end{equation}
Thus, we conclude the proof with
\begin{equation}
P(x_i+g(u_i) > x_j + g(u_j)) = \frac{1}{1+\exp{(-(x_i-x_j))}}.   
\end{equation}
\end{proof}
\end{lemma}

\begin{lemma}
\label{prop:concrete}
Let $X \sim \mbox{Gumbel}(\alpha, \tau)$ (i.e. $X_k=\frac{\exp((\log\alpha_k+g_k)/\tau)}{\sum_{i=1}^{n}\exp((\log\alpha_i+g_i)/\tau)}$) with location parameters $\alpha \in (0,\infty)^n$ and temperature $\tau \in (0, \infty)$, then:
\begin{itemize}
    \item $P(X_k > X_i, \forall i \neq k)=\frac{\alpha_k}{\sum_{i=1}^n \alpha_i}$,
    \item $P(\lim_{\tau\rightarrow0}X_k=1)=\frac{\alpha_k}{\sum_{i=1}^n \alpha_i}$.
\end{itemize}
\end{lemma}

\begin{proof}
This result can be similarly proved as Lemma~\ref{lem:gumbel}. The event of interests $X_k > X_i, \forall i \neq k$ is equivalent to
\begin{equation}
\label{eq:gumbel}
    \begin{aligned}
        \log\alpha_k - \log(-\log u_k) &> \log\alpha_1 - \log(-\log u_1), \\
        \log\alpha_k - \log(-\log u_k) &> \log\alpha_2 - \log(-\log u_2), \\
        &... \\
        \log\alpha_k - \log(-\log u_k) &> \log\alpha_n - \log(-\log u_n).
    \end{aligned}
\end{equation}
Since all the above inequalities are independent given $u_k$, we can rearrange the first inequality as
\begin{equation}
    u_1 < u_k^{\alpha_1/\alpha_k} \leq 1.
\end{equation}
Since $u_1 \sim U[0,1]$, the probability for the first inequality in Eqn.~\ref{eq:gumbel} being true would be $u_k^{\alpha_1/\alpha_k}$. Thus, the probability for Eqn.~\ref{eq:gumbel} being true can be calculated via
\begin{equation}
    u_k^{\alpha_1/\alpha_k}u_k^{\alpha_2/\alpha_k}...g_k^{\alpha_n/\alpha_k}=g_k^{(\alpha_1+\alpha_2+...+\alpha_n)/\alpha_k}=g_k^{(1/\alpha_k)-1}.
\end{equation}
For simplicity, we assume $\sum_{i=1}^n \alpha_i=1$. Then for any $g_k \in [0,1]$, we obtain
\begin{equation}
\begin{aligned}
    P(X_k > X_i, \forall i \neq k) &= \int_0^1 g_k^{(1/\alpha_k)-1}dg_k \\
    &= \frac{\alpha_k}{\sum_{i=1}^n \alpha_i},
\end{aligned}
\end{equation}
and arrive at the result for the first bullet point.
For the second bullet point, when $\tau\rightarrow0$, we have
\begin{equation}\label{eqn-proof-1-2}
\begin{aligned}
    &\lim_{\tau\rightarrow0}\frac{\exp((\log\alpha_i+g_i)/\tau)}{\exp((\log\alpha_j+g_j)/\tau)}\\
    =&\lim_{\tau\rightarrow0}\exp((\log\alpha_i+g_i-\log\alpha_j-g_j)/\tau) \\
    =&\left\{ 
    \begin{array}{cc}
         & \infty, \quad \mbox{if} \; \alpha_i > \alpha_j  \\
         & 0, \quad otherwise.
    \end{array}
    \right. 
    \end{aligned}
\end{equation}
Such a fact indicates that the output of a Concrete distribution with $\tau \rightarrow 0$ will be a one-hot vector ($X_{\arg\max_i\alpha_i}=1$). This yields the conclusion that
\begin{equation}
    P(\lim_{\tau\rightarrow0}X_k=1)=P(X_k > X_i, \forall i \neq k)=\frac{\alpha_k}{\sum_{i=1}^n \alpha_i}.
\end{equation}
\end{proof}

Now we turn to the proof of our theorem. We are to prove that the kernelized form in Eqn.~\ref{eqn-attn-gumbel-final} has the same property as the original Gumbel-Softmax in the limit sense (when $\tau$ goes to zero). We recall that we have defined $\mathbf q_u = W_Q^{(l)} \mathbf z_u^{(l)}$, $\mathbf k_u = W_K^{(l)} \mathbf z_u^{(l)}$ and $\mathbf v_u = W_V^{(l)} \mathbf z_u^{(l)}$ for simplicity.

First, by definition we have
\begin{equation}
\begin{aligned}
    &\phi(\frac{\mathbf q_u}{\sqrt{\tau}})^\top\phi(\frac{\mathbf k_v}{\sqrt{\tau}})e^{\frac{g_v}{\tau}}\\
    =&\frac{1}{m}\exp(-\frac{||\frac{\mathbf q_u}{\sqrt{\tau}}||^2+||\frac{\mathbf k_v}{\sqrt{\tau}}||^2}{2})\sum_{i=1}^m\exp(\omega_i^\top(\frac{\mathbf q_u}{\sqrt{\tau}}+\frac{\mathbf k_v}{\sqrt{\tau}}) + \frac{g_v}{\tau}).
\end{aligned}
\end{equation}
The property holds that for $\forall w \neq v$, we have $\lim_{\tau\rightarrow0}\frac{\phi(\frac{\mathbf q_u}{\sqrt{\tau}})^\top\phi(\frac{\mathbf k_v}{\sqrt{\tau}})e^{\frac{g_v}{\tau}}}{\phi(\frac{\mathbf q_u}{\sqrt{\tau}})^\top\phi(\frac{\mathbf k_w}{\sqrt{\tau}})e^{\frac{g_w}{\tau}}}$ equals to $\infty$ or 0, i.e. the output of the kernelized Gumbel-Softmax is still a one-hot vector when $\tau\rightarrow0$. Let 
\begin{equation}
    Y_v = \frac{\phi(\frac{\mathbf q_u}{\sqrt{\tau}})^\top\phi(\frac{\mathbf k_v}{\sqrt{\tau}})e^{\frac{g_v}{\tau}}}{\sum_{w=1}^N \phi(\frac{\mathbf q_u}{\sqrt{\tau}})^\top\phi(\frac{\mathbf k_w}{\sqrt{\tau}})e^{\frac{g_w}{\tau}}}.
\end{equation} 
Here $Y_v$ is defined in the same way as $c_{uv}$ in Section~\ref{sec-model-theory}. We thus have $P(\lim_{\tau\rightarrow0}Y_v=1)=P(Y_v > Y_{v'}, \forall v' \neq v)$.

To compute $P(Y_v > Y_{v'}, \forall v'\neq v)$, for simplicity, let us consider the probability $P(Y_v > Y_{v'})=P(\phi(\frac{\mathbf q_u}{\sqrt{\tau}})^\top\phi(\frac{\mathbf k_v}{\sqrt{\tau}})e^{\frac{g_v}{\tau}}>\phi(\frac{\mathbf q_u}{\sqrt{\tau}})^\top\phi(\frac{\mathbf k_{v'}}{\sqrt{\tau}})e^{\frac{g_{v'}}{\tau}})$. To keep notation clean, we define
\begin{equation}
    \beta_v=\phi(\frac{\mathbf q_u}{\sqrt{\tau}})^\top\phi(\frac{\mathbf k_v}{\sqrt{\tau}}),\  \beta_{v'}=\phi(\frac{\mathbf q_u}{\sqrt{\tau}})^\top\phi(\frac{\mathbf k_{v'}}{\sqrt{\tau}}).
\end{equation}
Then the above-mentioned probability can be rewritten as $P(\log\beta_v+\frac{g_v}{\tau} > \log\beta_{v'}+\frac{g_{v'}}{\tau})$, where $\beta_v$ and $\beta_{v'}$ are two i.i.d. random variables.

From Lemma~\ref{lem:kernel}, we have $\mbox{E}(\beta_v)= \exp(\mathbf q_u^\top \mathbf k_v / \tau) = \alpha_v^{\frac{1}{\tau}},\ \mbox{E}(\beta_{v'})= \exp(\mathbf q_u^\top \mathbf k_{v'} / \tau) = \alpha_{v'}^{\frac{1}{\tau}}$, where $\alpha_v$ and $\alpha_{v'}$ are two constant values. Then using Lemma~\ref{lem:gumbel}, we have
\begin{equation}
    \begin{aligned}
    &P(\log\alpha_v^{1/\tau}+\frac{g_v}{\tau}>\log\alpha_{v'}^{1/\tau}+\frac{g_{v'}}{\tau}) \\
    = & P(\log\alpha_v+g_v>\log\alpha_{v'}+g_{v'}) \\
    = &\frac{1}{1+\exp(\log\alpha_{v'}-\log\alpha_v)} \\
    = &\frac{\alpha_{v'}}{\alpha_v+\alpha_{v'}}.
    \end{aligned}
\end{equation}

According to the Chebyshev's inequality, we have $P(|\beta_v-\alpha_v^{\frac{1}{\tau}}|\leq\epsilon_v)\geq 1-\frac{\sigma_v^2}{\epsilon_v^2}$. Here $\sigma_v^2=\mathbb V_{\mathbf w}(\widehat{\mbox{SM}}_m(\frac{\mathbf q_u}{\sqrt{\tau}},\frac{\mathbf k_j}{\sqrt{\tau}}))$, which can given by Lemma~\ref{lem:kernel}. 

Due to the convexity of logarithmic function, we have
\begin{equation}
    \frac{|\log\beta_v-\frac{1}{\tau}\log\alpha_v|}{|\beta_v-\alpha_v^{\frac{1}{\tau}}|} \leq \frac{1}{\alpha_v^{\frac{1}{\tau}}-\epsilon_v},
\end{equation}
and subsequently,
\begin{equation}
    \begin{aligned}
    |\log\beta_v-\frac{1}{\tau}\log\alpha_v| &\leq \frac{|\beta_v-\alpha_v^{\frac{1}{\tau}}|}{\alpha_v^{\frac{1}{\tau}}-\epsilon_v} \\ 
    &\leq \frac{\epsilon_v}{\alpha_v^{\frac{1}{\tau}}-\epsilon_v}.
    \end{aligned}
\end{equation}
Therefore we have $P(|\log\beta_v-\frac{1}{\tau}\log\alpha_v| \leq \frac{\epsilon_v}{\alpha_v^{\frac{1}{\tau}}-\epsilon_v}) \geq P(|\beta_v-\alpha_v^{\frac{1}{\tau}}|\leq\epsilon_v)$. Based on this, we can derive the result:
\begin{equation}
    P(|\log\beta_v-\frac{1}{\tau}\log\alpha_v| \leq \frac{\epsilon_v}{\alpha_v^{\frac{1}{\tau}}-\epsilon_v}) \geq 1 - \frac{\sigma_v^2}{\epsilon_v^2}.
\end{equation}
Since $\beta_v$ and $\beta_{v'}$ are two i.i.d. random variables, we have
\begin{equation}
\begin{aligned}
    P(|\log\beta_v-\frac{1}{\tau}\log\alpha_v| \leq \frac{\epsilon_v}{\alpha_v^{\frac{1}{\tau}}-\epsilon_v}&,\\ 
    |\log\beta_{v'}-\frac{1}{\tau}\log\alpha_{v'}| \leq \frac{\epsilon_{v'}}{\alpha_{v'}^{\frac{1}{\tau}}-\epsilon_{v'}})
    &\geq (1 - \frac{\sigma_v^2}{\epsilon_v^2})(1 - \frac{\sigma_{v'}^2}{\epsilon_{v'}^2}).
\end{aligned}
\end{equation}
For simplicity, we denote $\epsilon = \frac{\epsilon_v}{\alpha_v^{\frac{1}{\tau}}-\epsilon_v}+\frac{\epsilon_{v'}}{\alpha_{v'}^{\frac{1}{\tau}}-\epsilon_{v'}}$ and $P_\epsilon=(1 - \frac{\sigma_v^2}{\epsilon_v^2})(1 - \frac{\sigma_{v'}^2}{\epsilon_{v'}^2})$. We therefore have
\begin{equation}\label{eqn-proof-2-1}
    P(|\log\beta_v-\frac{1}{\tau}\log\alpha_v|+|\log\beta_{v'}-\frac{1}{\tau}\log\alpha_{v'}| \leq \epsilon) \geq P_\epsilon.
\end{equation}
Using the triangular inequality, we can yield
\begin{equation}\label{eqn-proof-2-2}
    \begin{aligned}
    |(\log\beta_v-\frac{1}{\tau}\log\alpha_v)-(\log\beta_{v'}-\frac{1}{\tau}\log\alpha_{v'})| \\ 
    \leq |\log\beta_v-\frac{1}{\tau}\log\alpha_v|+|\log\beta_{v'}-\frac{1}{\tau}\log\alpha_{v'}|.
    \end{aligned}
\end{equation}
Combining Eqn.~\ref{eqn-proof-2-1} and \ref{eqn-proof-2-2}, we have
\begin{equation}
\label{eq:ineq}
    P(|(\log\beta_v-\frac{1}{\tau}\log\alpha_v)-(\log\beta_{v'}-\frac{1}{\tau}\log\alpha_{v'})| \leq \epsilon) \geq P_\epsilon.
\end{equation}
Let $c=\log\beta_v-\log\beta_{v'}$, so that $\mbox{E}(c)=\frac{1}{\tau}(\log\alpha_v-\log\alpha_{v'})$. From Eqn.~\ref{eq:ineq}, we can obtain
\begin{equation}\label{eqn-proof-2-3}
    P(c \geq \mbox{E}(c)-\epsilon) \geq P_\epsilon.
\end{equation}
According to Lemma~\ref{lem:gumbel}, the probability $P(\mbox{E}(c)-\epsilon \geq \frac{g_{v'}-g_v}{\tau})$ can be written as
\begin{equation}\label{eqn-proof-2-4}
\begin{aligned}
    &P(\log\alpha_v-\log\alpha_{v'}-\tau\epsilon \geq g_{v'}-g_v) \\
    &= \frac{1}{1+\exp(\log\alpha_{v'}-\log\alpha_v+\tau\epsilon)} \\
    &= \frac{1}{1+\frac{\alpha_{v'}}{\alpha_v}e^{\tau\epsilon}}.
\end{aligned}
\end{equation}
Since $c, g_v, g_{v'}$ are generated independently, combining Eqn.~\ref{eqn-proof-2-3} and \ref{eqn-proof-2-4}, we can yield
\begin{equation}\label{eqn-proof-2-5}
    P(c \geq \frac{g_{v'}-g_v}{\tau}) \geq \frac{P_\epsilon}{1+\frac{\alpha_{v'}}{\alpha_v}e^{\tau\epsilon}}.
\end{equation}
Similarly, from Eq.~\ref{eq:ineq} we have $P(c\leq \mbox{E}(c)+\epsilon) \geq P_\epsilon$ and subsequently,
\begin{equation}
\begin{aligned}
    P(\frac{g_{v'}-g_v}{\tau} \geq \mbox{E}(c)+\epsilon)&=1-P(c+\epsilon \geq \frac{g_{v'}-g_v}{\tau}) \\
    &= 1 - \frac{1}{1+\frac{\alpha_{v'}}{\alpha_v}e^{-\tau\epsilon}}.
\end{aligned}
\end{equation}
Thus we have $P(c \leq \frac{g_{v'}-g_v}{\tau}) \geq P_\epsilon(1-\frac{1}{1+\frac{\alpha_{v'}}{\alpha_v}e^{-\tau\epsilon}})$ and also
\begin{equation}\label{eqn-proof-2-6}
    P(c \geq \frac{g_{v'}-g_v}{\tau}) \leq 1-P_\epsilon(1-\frac{1}{1+\frac{\alpha_{v'}}{\alpha_v}e^{-\tau\epsilon}}).
\end{equation}
Combining Eqn.~\ref{eqn-proof-2-5} and \ref{eqn-proof-2-6}, we conclude that
\begin{equation}
\label{eq:conclusion}
    \frac{P_\epsilon}{1+\frac{\alpha_{v'}}{\alpha_v}e^{\tau\epsilon}} \leq P(c \geq \frac{g_{v'}-g_v}{\tau}) \leq 1-P_\epsilon(1-\frac{1}{1+\frac{\alpha_{v'}}{\alpha_v}e^{-\tau\epsilon}}).
\end{equation}
Based on this we consider the limitation for two sides and thus obtain
\begin{equation}
    \lim_{P_\epsilon \rightarrow 1}\lim_{\tau\rightarrow 0}P(c \geq \frac{g_{v'}-g_v}{\tau}) = \frac{1}{1+\frac{\alpha_{v'}}{\alpha_v}}=\frac{\alpha_v}{\alpha_v+\alpha_{v'}}.
\end{equation}

Then with similar reasoning as Lemma~\ref{prop:concrete}, we have
\begin{equation}
\begin{aligned}\label{eqn-proof-2-7}
    &\lim_{P_\epsilon \rightarrow 1}\lim_{\tau\rightarrow 0}P(Y_v=1) \\
    =&\lim_{P_\epsilon \rightarrow 1}\lim_{\tau\rightarrow 0}P(Y_v > Y_{v'},\forall v' \neq v)=\alpha_v/(\sum_{w=1}^N\alpha_w).
\end{aligned}
\end{equation}
Recall that
\begin{equation}
    \begin{aligned}
    P_\epsilon&=(1-\frac{\sigma_v^2}{\epsilon_v^2})(1-\frac{\sigma_{v'}^2}{\epsilon_{v'}^2}) \\
    \sigma^2&=\mathbb V_{\mathbf w}(\widehat{SM}_m^+(\mathbf x,\mathbf y))\\
    &=\frac{1}{m}\exp(-\frac{\|\mathbf x\|^2+\|\mathbf y\|^2}{2})\sum_{i=1}^m \exp(\mathbf w_i^\top(\mathbf x+\mathbf y)),
    \end{aligned}
\end{equation}
where $\mathbf x=\frac{\mathbf q_u}{\sqrt{\tau}}, \mathbf y=\frac{\mathbf k_{v,v'}}{\sqrt{\tau}}$. Therefore, $P_\epsilon$ is dependent of the precision $\epsilon_v,\epsilon_{v'}$, the random feature dimension $m$, and the temperature $\tau$. If $m$ is sufficiently large, $\sigma$ would converge to zero and $P_\epsilon$ goes to 1. In such a case, Eqn.~\ref{eqn-proof-2-7} holds once $\tau$ tends to zero. We thus conclude the proof.



\subsection{Proof for Proposition~\ref{prop-var}}
According to our definitions in Section~\ref{sec-analysis}, we have
\begin{equation}
    \begin{split}
        & \mathcal D_{KL}(q_\phi(\tilde{\mathbf A}|\mathbf X, \mathbf A) \| p(\tilde{\mathbf A}|\mathbf Y, \mathbf X, \mathbf A)) \\
        = & \int_{A^*} q_\phi(\tilde{\mathbf A}|\mathbf X, \mathbf A) \log \frac{q_\phi(\tilde{\mathbf A}|\mathbf X, \mathbf A)}{p(\tilde{\mathbf A}|\mathbf Y, \mathbf X, \mathbf A)} d\tilde{\mathbf A}\\
        = & \int_{A^*} q_\phi(\tilde{\mathbf A}|\mathbf X, \mathbf A) \log \frac{q_\phi(\tilde{\mathbf A}|\mathbf X, \mathbf A)p_\theta(\mathbf Y| \mathbf X, \mathbf A)}{p(\tilde{\mathbf A}, \mathbf Y| \mathbf X, \mathbf A)} d\tilde{\mathbf A}\\
        = & \int_{A^*} q_\phi(\tilde{\mathbf A}|\mathbf X, \mathbf A) \log \frac{q_\phi(\tilde{\mathbf A}|\mathbf X, \mathbf A)p_\theta(\mathbf Y| \mathbf X, \mathbf A)}{p(\tilde{\mathbf A}, \mathbf Y| \mathbf X, \mathbf A)} d\tilde{\mathbf A}\\
        = & \int_{A^*} q_\phi(\tilde{\mathbf A}|\mathbf X, \mathbf A) \log \frac{q_\phi(\tilde{\mathbf A}|\mathbf X, \mathbf A)p_\theta(\mathbf Y| \mathbf X, \mathbf A)}{p(\mathbf Y| \tilde{\mathbf A}, \mathbf X, \mathbf A) p(\tilde{\mathbf A}|\mathbf X, \mathbf A)} d\tilde{\mathbf A}\\
        = & - \mathbb E_{q_\phi(\tilde{\mathbf A}|\mathbf X, \mathbf A)} [\log p(\mathbf Y| \tilde{\mathbf A}, \mathbf X, \mathbf A)] + \log p_\theta(\mathbf Y| \mathbf X, \mathbf A) +  \mathcal D_{KL}(q_\phi(\tilde{\mathbf A}|\mathbf X, \mathbf A) \| p(\tilde{\mathbf A}|\mathbf X, \mathbf A)) \\
        = & - \mbox{ELBO}(\theta, \phi) + \log p_\theta(\mathbf Y| \mathbf X, \mathbf A)
    \end{split}
\end{equation}
Since we assume $q_\phi$ can exploit arbitrary distributions over $\tilde{\mathbf A}$, we know that when the ELBO is optimized to the optimum, $\mathcal D_{KL}(q_\phi(\tilde{\mathbf A}|\mathbf X, \mathbf A) \| p(\tilde{\mathbf A}|\mathbf Y, \mathbf X, \mathbf A)) = 0$ holds. Otherwise, there exists $\phi^* \neq \phi$ such that $\mbox{ELBO}(\theta, \phi^*) > \mbox{ELBO}(\theta, \phi)$. Pushing further, when the optimum is achieved, $\log p_\theta(\mathbf Y| \mathbf X, \mathbf A)$ would equal to $\mbox{ELBO}$ and namely is maximized.

\section{Implementation Details}\label{appx-implementation}
We present implementation details in our experiments for reproducibility. For more concrete details concerning architectures and hyper-parameter settings for \model, one can directly refer to our public repository \url{https://github.com/qitianwu/NodeFormer}. We next present descriptions for baseline models' implementation. 
For baseline models MLP, GCN, GAT, MixHop and JKnet, we use the implementation provided by~\cite{newbench}\footnote{https://github.com/CUAI/Non-Homophily-Benchmarks.}.
For DropEdge and two structure learning baseline models (LDS and IDGL), we also refer to their implementation provided by the original papers~\cite{dropedge-iclr20, IDLS-icml19, IDGL-neurips20}. Concretely, we use GCN as the backbone for them.

\subsection{Details for Node Classification Experiments in Sec.~\ref{sec-exp-trans}}
\textbf{Architectures.} For experiments on the datasets \texttt{Cora}, \texttt{Citeseer}, \texttt{Deezer} and \texttt{Actor}, the baseline models (GCN, GAT, MixHop, JKNet) are implemented with the following settings:
\begin{itemize}[leftmargin=*,itemsep=0pt,topsep=0pt]
    \item Two GNN layers with hidden size 32 by default (unless otherwise mentioned). GAT uses 8 attention heads followed by its original setting.
    \item The activation function is ReLU (except GAT using ELU).
\end{itemize}

The architecture of our \model is specified as follows:
\begin{itemize}[leftmargin=*,itemsep=0pt,topsep=0pt]
    \item Two message-passing layers with hidden size 32. We also consider multi-head designs for our all-pair attentive message passing, and for each head we use independent parameterization. The results for different heads are combined in an average manner in each layer. 
    \item The activation function is ELU that is only used for input MLP, and we do not use any activation for the feature propagation layers. In terms of relational bias, we specify $\sigma$ as sigmoid function and consider 2-order adjacency to strengthen the observed links of the input graph.
\end{itemize}

\textbf{Training Details.} In each epoch, we feed the whole data into the model, calculate the loss and conduct gradient descent accordingly. Concretely, we use BCE-Loss for two-class classification and NLL-Loss for multi-class classification, the Adam optimizer is used for gradient-based optimization. The training procedure will repeat the above process until a given budget of 1000 epochs. Finally, we report the test accuracy achieved by the epoch that gives the highest accuracy on validation dataset.

\textbf{Hyperparameters.} For each model, we use grid search on validation set for hyper-parameter setting. The learning rate is searched within $\{0.01, 0.001, 0.0001, 0.00001\}$, weight decay within $\{0.05, 0.005, 0.0005, 0.00005\}$, and dropout probability within $\{0.0, 0.5\}$. The hyper-parameters for \model is provided in our public codes. The hyperparameters for baseline models are set as follows (we use the same notation as the original papers).
\begin{itemize}[leftmargin=*,itemsep=0pt,topsep=0pt]
    \item For GCN and GAT, the learning rate is 0.01, and weight decay is set to 0.05. No dropout is used.
    \item For MixHop, the hyperparameters are the same as above, except that we further use grid search for hidden channels within \{8, 16, 32, 64, 128\}. We adopt 2 hops for all the four datasets.
    \item For JKNet, GCN is used as the backbone. Learning rate is set to 0.01 for \texttt{Deezer} and 0.001 for all the other three datasets. We concatenate the features in the final stage for Deezer, while we use max-pooling for the three other datasets. The hidden size is set as default, except for Deezer as 64.
    \item For DropEdge, the hidden size is chosen from \{32, 64, 96, 128, 160\}, the learning rate is within \{0.01, 0.001, 0.0001\}, and the dropedge rate is chosen from \{0.3, 0.4, 0.5\}.
    \item For LDS, the sampling time $S=16$, the patience window size $\rho=6$, the hidden size $\in$ \{8, 16, 32, 64\}, 
    the inner learning rate $\gamma \in$ \{1e-4, 1e-3, 1e-2, 1e-1\}, 
    and the number of updates used to compute the truncated hypergradient $\tau \in$ \{5, 10, 15\}. 
    \item For IDGL, we use its original version without anchor approximation on \texttt{Cora}, \texttt{Citeseer} and \texttt{Actor}. For \texttt{Deezer}, even using anchor approximation, it would also suffer from out-of-memory. Besides, we set: $\epsilon=0.01$, hidden size $\in$ \{16, 64, 96, 128\}, $\lambda \in$ \{0.5, 0.6, 0.7, 0.8\}, $\eta \in$ \{0, 0.1, 0.2\}, $\alpha \in$ \{0, 0.1, 0.2\}, $\beta \in$ \{0, 0.1\}, $\gamma \in$ \{0.1, 0.2\}, $m \in$ \{6, 9, 12\}.
\end{itemize}

\subsection{Details for Node Classification on Larger Graphs in Sec.~\ref{sec-exp-large}}

\textbf{Architectures.} For experiments on the two large datasets (OGB-Proteins and Amazon2M), the baseline models are implemented with the following settings:
\begin{itemize}[leftmargin=*,itemsep=0pt,topsep=0pt]
    \item Three GNN layers with hidden size 64.
    \item The activation function is ReLU (except GAT using ELU).
\end{itemize}

The architecture of our \model is specified as follows:
\begin{itemize}[leftmargin=*,itemsep=0pt,topsep=0pt]
    \item Three message-passing layers with hidden size 64. The head number is set as 1. 
    \item The activation function is ELU that is used for all the layers. In terms of relational bias, we specify $\sigma$ as identity function and consider 1-order adjacency to strengthen the observed links of the input graph.
\end{itemize}

\textbf{Training Details.} In each epoch, we use random mini-batch partition to split the whole set of nodes and feed each mini-batch of nodes into the model for all-pair propagation, as we mentioned in Section~\ref{sec-exp-trans}. Similarly, we use BCE-Loss for two-class classification and NLL-Loss for multi-class classification, the Adam optimizer is used for gradient-based optimization. The training procedure will repeat the above process until a given budget of 1000 epochs. The evaluation on testing data is conducted on CPU which enables full-batch feature propagation. Finally, we report the test accuracy achieved by the epoch that gives the highest accuracy on validation dataset.

\subsection{Details for Graph-Enhanced Experiments in Sec.~\ref{sec-exp-app}}
\textbf{Architectures.} The architectures of baselines (GCN, GAT, LDS and IDGL) and \model model are the same as the transductive setting, except that we use grid search to adaptively tune the hidden size. Besides, we also adopt BatchNorm for baseline models.

\textbf{Training Details.} The input data have no graph structures in this setting. As mentioned in Section~\ref{sec-exp-app}, we use $k$-NN for artificially constructing a graph to enable message passing. The training procedure is the same as the transductive setting.

\textbf{Hyperparameters.} The hyperparameters for baseline models are listed as follows.
\begin{itemize}[leftmargin=*,itemsep=0pt,topsep=0pt]
    \item For GCN, the learning rate is 0.01 and the weight decay is 5e-4 on both datasets. The size of hidden channel is set to 64. Dropout is not used during training.
    \item For GAT, the learning rate is 0.001 and the weight decay is 5e-3 on both datasets. The size of hidden channel is 32 on \texttt{Mini-ImageNet} and 64 on \texttt{20News}. No dropout is used during training. The number of attention heads is 8.
    \item For LDS, its hyperparameters are determined by grid search in the same manner as in the transductive setting.
    \item For IDGL, we use its anchor-based version that can scale to these two datasets. Besides, on \texttt{20News}, we set: hidden size=64, learning rate=0.01 $\lambda=0.7$, $\eta=0.1$, $\alpha=0.1$, $\beta=0$, $\gamma=0.1$, $\epsilon=0.01$, $m=12$. On \texttt{Mini-ImageNet}, we set: hidden size=96, learning rate=0.01 $\lambda=0.8$, $\eta=0.2$, $\alpha=0$, $\beta=0$, $\gamma=0.1$, $\epsilon=0.01$, $m=12$.
\end{itemize}

\section{Dataset Information}\label{appx-dataset}

We present detailed information for our used datasets concerning the data collection, preprocessing and statistic information. Table~\ref{tab:dataset} provides an overview of the datasets we used in the experiment.

\subsection{Dataset Information}

\begin{table*}[htbp]
    \centering
    \caption{Information for experiment datasets.}
    \label{tab:dataset}
    \resizebox{0.9\textwidth}{!}{
    \begin{tabular}{lccccccc}
    \hline
    Dataset & Context & Property  & \# Task & \# Nodes & \# Edges & \# Node Feat & \# Class  \\ \hline
    Cora & Citation network & homophilous & 1 & 2,708 & 5,429 & 1,433 & 7 \\
    Citeseer & Citation network & homophilous & 1 & 3,327 & 4,732 & 3,703 & 6 \\
    Deezer & Social network & non-homophilous & 1 & 28,281 & 92,752 & 31,241 & 2 \\
    Actor & Actors in movies & non-homophilous & 1 & 7,600 & 29,926 & 931 & 5 \\ \hline
    OGB-Proteins & Protein interaction & multi-task classification & 112 & 132,534 & 39,561,252 & 8  & 2 \\ 
    Amazon2M & Product co-occurrence & long-range dependence & 1 & 2,449,029 & 61,859,140 & 100  & 47 \\ \hline
    Mini-ImageNet & Image classification & no graph/$k$-NN graph & 1 & 18,000 & 0 & 128 & 30 \\
    20News-Groups & Text classification & no graph/$k$-NN graph & 1 & 9,607 & 0 & 236 & 10 \\ \hline
    \end{tabular}
    }
\end{table*}

\textbf{Node Classification Datasets.} For experiments on transductive node classification, we evaluate our model on two homophilous datasets \texttt{Cora} and \texttt{Citeseer}~\cite{Sen08collectiveclassification}, and other two non-homophilous datasets \texttt{Actor}~\cite{geomgcn-iclr20} and \texttt{Deezer}~\cite{deezer-cikm19}. The first two are citation network datasets that contain sparse bag-of-words feature vectors for each document and a list of citation links between documents. The citation links are treated as (undirected) edges and each document has a class label. \texttt{Deezer} is a social network of users on Deezer from European countries, where edges represent mutual follower relationships. The node features are based on artists liked by each user and nodes are labeled with reported gender. \texttt{Actor} is a graph representing actor co-occurrence in Wikipedia pages. Each node corresponds to an actor, and the edge between two nodes denotes co-occurrence on the same Wikipedia page. Node features correspond to some keywords in the Wikipedia pages and the nodes are classified into five categories w.r.t. words of actor's Wikipedia. These four datasets are relatively small with thousands of instances and edges (\texttt{Deezer} is the largest one with nearly 20K nodes).

To evaluate \model's scalability, we further consider two large datasets: \texttt{OGB-Proteins}~\cite{ogb-nips20} and \texttt{Amazon2M}~\cite{clustergcn-kdd19}. These two datasets have million-level nodes and edges and require the model for scalable training. The \texttt{OGB-Proteins} dataset is an undirected, and typed (according to species) graph in which nodes represent proteins and edges indicate different types of biologically meaningful associations between proteins. All edges come with 8-dimensional features, where each dimension represents the approximate confidence of a single association type and takes on values between 0 and 1. The proteins come from 8 species and our task is to predict the presence of 112 protein functions in a multi-label binary classification setup respectively. \texttt{Amazon2M} is extracted from Amazon Co-Purchasing network~\cite{amazoncopurchase-kdd15}, where each node represents a product, and the graph link represents whether two products are purchased together, the node features are generated by extracting bag-of-word features from the product descriptions. The top-level categories are used as labels for the products.


\textbf{Graph-enhanced Application Datasets.} We evaluate our model on two datasets without graph structure: \texttt{20News-Groups}~\cite{pedregosa2011scikit} and \texttt{Mini-ImageNet}~\cite{miniimagenet-neurips2016}. The \texttt{20News} dataset is a collection of approximately 20,000 newsgroup documents (nodes), partitioned (nearly) evenly across 20 different newsgroups. We take 10 classes from 20 newsgroups and use words (TFIDF) with a frequency of more than 5\% as features. The \texttt{Mini-ImageNet} dataset consists of 84$\times$84 RGB images from 100 different classes with 600 samples per class. For our experiment use, we choose 30 classes from the dataset, each with 600 images (nodes) that have 128 features extracted by CNN.

\subsection{Dataset Preprocessing}
All the datasets we used in the experiment are directly collected from the source, except \texttt{Mini-ImageNet}, whose features are extracted by ourselves. Following the setting of \cite{garcia2018fewshot}, we compute node embeddings via a CNN model with 4 convolutional layers followed by a fully-connected layer resulting in a 128 dimensional embedding.
Finally, the 128 dimensional outputs are treated as the features of the nodes (images) for subsequent GNN-based downstream task.

\section{More Experiment Results}\label{appx-result}

We present extra ablation study results on the four transductive datasets for \model w/ and w/o relational bias and edge-level regularization. Fig.~\ref{fig:mt_ablation} studies the impact of the temperature $\tau$ and the dimension of feature map $m$ on \texttt{Cora}. Furthermore, 
we visualize the attention maps of two model layers and compare with original input graphs of \texttt{Cora}, \texttt{Citeseer}, \texttt{Deezer} and \texttt{Actor} in Fig.~\ref{fig:vis-att}.


\begin{table}[htbp]
    \centering
    \caption{Ablation study results on transductive datasets, where ``rb'' denotes relational bias and ``reg'' represents edge-level regularization.}
    \label{tab:ablation}
    \resizebox{0.6\linewidth}{!}{
    \begin{tabular}{l|ccc}
    \toprule
    \textbf{Dataset} & \model & \model w/o reg & \model w/o rb  \\ 
    \midrule
    \texttt{Cora} & \textbf{88.69} \std{$\pm$ 0.46} & 81.98 \std{$\pm$ 0.46} & 88.06 \std{$\pm$ 0.59} \\
    \texttt{Citeseer} & \textbf{76.33} \std{$\pm$ 0.59} & 70.60 \std{$\pm$ 1.20} & 74.12 \std{$\pm$ 0.64} \\
    \texttt{Deezer} & \textbf{71.24} \std{$\pm$ 0.32} & 71.22 \std{$\pm$ 0.32} & 71.10 \std{$\pm$ 0.36} \\
    \texttt{Actor} & \textbf{35.31} \std{$\pm$ 1.29} & 35.15 \std{$\pm$ 1.32} & 34.60 \std{$\pm$ 1.32} \\ 
    \bottomrule
    \end{tabular}}
\end{table}

\begin{figure}[htbp]
    \centering
    \includegraphics[width=0.7\linewidth]{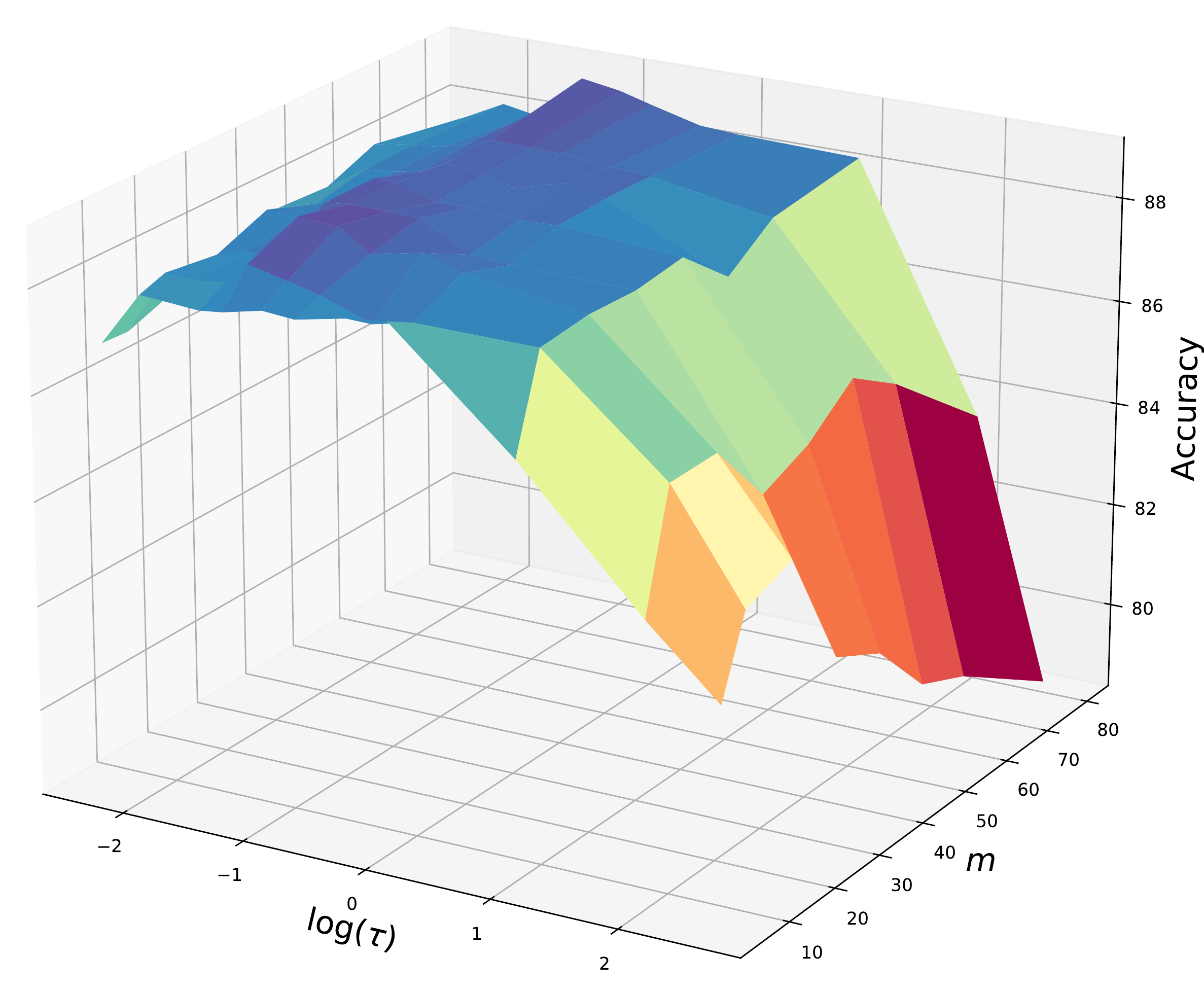}
    \caption{Impact of the temperature $\tau$ and the dimension of random feature map $m$ on \texttt{Cora}.}
    \label{fig:mt_ablation}
\end{figure}

\begin{figure}[htbp]
    \centering
    \includegraphics[width=\textwidth]{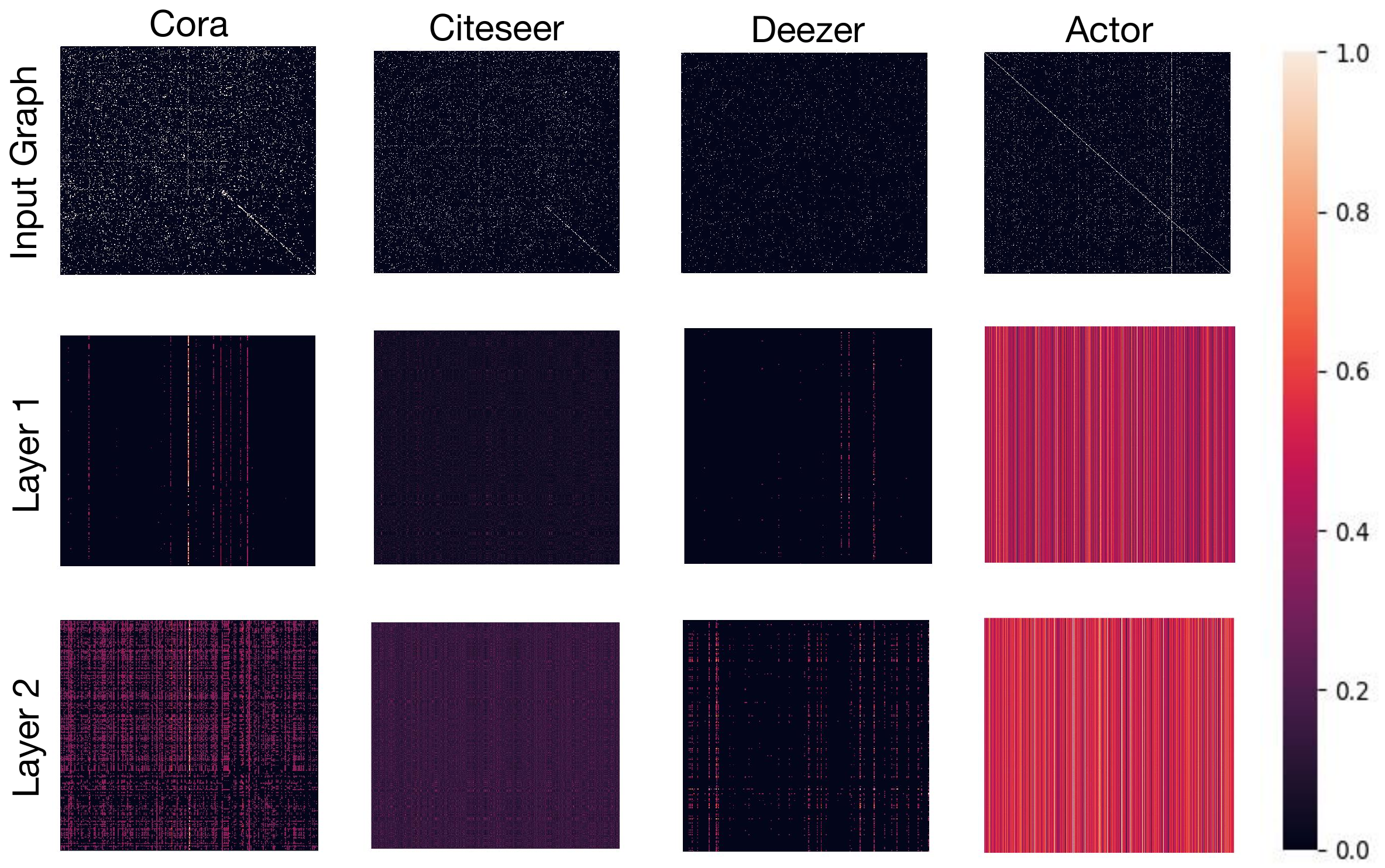}
    \caption{Visualization for input graph structures and latent graph structures (given by two layers of \model) with colors reflecting the weights.}
    \label{fig:vis-att}
\end{figure}

\section{Current Limitations, Outlooks and Potential Impacts}\label{appx-impact}

\textbf{Current Limitations.} In the present work, we focus on node classification for experiments, though \model can be used as a flexible (graph) encoder for other graph-related problems such as graph classification, link prediction, etc. Beyond testing accuracy, social aspects like robustness and explainability can also be considered as the target for future works on top of \model.

\textbf{Potential Impact.} Besides facilitating better node representations via message passing, graph structure learning also plays as key components in many other perpendicular problems in graph learning community, like explainability~\cite{gnnexplain-neurips19}, knowledge transfer and distillation~\cite{geokd}, adversarial robustness~\cite{gnnguard-neurips20}, training acceleration~\cite{fastgat-20}, handling feature extrapolation~\cite{wu2021feature} and cold-start users in recommendation~\cite{wu2021rec}. \model can serve as a plug-in scalable structure learning encoder for uncovering underlying dependence, identifying novel structures and purifying noisy data in practical systems. Another promising direction is to leverage our kernelized Gumbel-Softmax operator as a plug-in module for designing efficient and expressive Transformers on graph data where the large graph size plays a critical performance bottleneck. 

\end{document}